\documentclass[]{malab}
\usepackage{wrapfig}
\usepackage{hyperref}
\definecolor{citeblue}{RGB}{35, 140, 190}
\definecolor{refred}{RGB}{210, 70, 60}

\hypersetup{
    colorlinks=true,
    citecolor=citeblue,
    linkcolor=refred,
    urlcolor=citeblue
}
\usepackage{url}
\usepackage{pifont}
\usepackage{amsmath}
\usepackage{tabularx}
\usepackage{mathtools}
\usepackage{bm}
\usepackage{multirow}
\usepackage{booktabs}
\usepackage{colortbl}
\usepackage{xcolor}
\usepackage{graphicx}
\usepackage{pgf}
\usepackage{caption}
\usepackage[most]{tcolorbox}

\newcommand{\cmark}{\ding{51}}
\newcommand{\xmark}{\ding{55}}

\DeclareMathOperator*{\argmax}{arg\,max}
\definecolor{exampleborder}{RGB}{48,68,86}
\definecolor{exampletitle}{RGB}{48,68,86}
\definecolor{examplebg}{RGB}{248,249,250}
\definecolor{gaincolor}{RGB}{40,125,125}  
\definecolor{dropcolor}{RGB}{180,85,95}   
\definecolor{samecolor}{RGB}{120,120,120} 
\newtcolorbox{qualexample}[1]{
    enhanced,
    colback=examplebg,
    colframe=exampleborder,
    coltitle=white,
    colbacktitle=exampletitle,
    fonttitle=\bfseries,
    title={#1},
    boxed title style={
        sharp corners,
        boxrule=0pt,
    },
    attach boxed title to top left={
        xshift=7mm,
        yshift=-2.5mm
    },
    top=6mm,
    bottom=3mm,
    left=4mm,
    right=4mm,
    boxrule=0.8pt,
    arc=2.5mm,
    before skip=8pt,
    after skip=8pt,
    breakable=false,
}
\definecolor{appendixblue}{RGB}{28,63,130}

\theoremstyle{definition}

\theoremstyle{plain}

\title{Learning an Anchored Prompt Space for Continual Adaptation of Large Language Models}

\shorttitle{Learning an Anchored Prompt Space for Continual Adaptation of Large Language Models}

\author[1,2]{Rongguang Ye}
\author[1]{Zhan Zhuang}
\author[3]{Yichen Wu}
\author[2]{Ming Tang}
\author[1]{Kede Ma}
\affiliation[1]{City University of Hong Kong}
\affiliation[2]{Southern University of Science and Technology}
\affiliation[3]{Harvard University }

\abstract{
Continually adapting large language models requires acquiring new knowledge while
preserving previously learned capabilities. Although jointly adapting model parameters and task-specific soft prompts is promising,
it has two key limitations: previously learned prompts may lose effectiveness as
the model evolves, and relationships among prompts that could support cross-task
transfer are not explicitly learned. We propose Learning an Anchored Prompt Space
(LAPS), which preserves historical prompt effectiveness and learns relationships
among task-specific soft prompts to facilitate positive transfer.
Historical prompts are first aligned with the updated model through
self-distillation. LAPS then constructs an anchored prompt space whose vertices correspond to learned task-specific soft prompts and whose intermediate geometry is shaped by learnable B\'ezier control prompts. Once
this space is learned, LAPS identifies the best-performing
prompt for each task on its validation set, allowing the selected prompt to
draw on knowledge acquired from related tasks.
Experiments on the TRACE benchmark across three Qwen3 model scales show that
LAPS consistently outperforms existing distillation-based, prompt-based, and joint prompt-parameter adaptation methods, improving overall task performance while
reducing forgetting.
}

\date{\today}

\authoremails{
\href{mailto:rongguaye2-c@my.cityu.edu.hk}{\texttt{\{rongguaye2-c,zhazhuang3-c\}@my.cityu.edu.hk}},
\email{yiwu6@mgh.harvard.edu}
}
\correspondence{\email{tangm3@sustech.edu.cn}, \email{kede.ma@cityu.edu.hk}}

\projectpage{\url{https://github.com/Multimedia-Analytics-Laboratory/LAPS}}

\begin{document}

\maketitle

\section{Introduction}
Biological learning achieves efficient adaptation by combining plastic mechanisms that incorporate
new evidence with stabilizing mechanisms that protect useful prior structure, rather than rewriting an
entire system for each new experience
\citep{grossberg1980adaptive,mccloskey1989catastrophic,french1999catastrophic}.
Complementary learning systems provide
classic accounts of how learning can balance acquisition and retention
\citep{mcclelland1995why,kirkpatrick2017overcoming}. A similar pressure arises
in large language model (LLM) deployment: after pretraining and instruction
tuning, LLMs must adapt to new tasks, domains, policies, and user requirements
without full retraining
\citep{wang2023orthogonal,ibrahim2024simple,yang2025recent}. Practical
adaptation must incorporate task-specific evidence from limited data while
retaining capabilities acquired from previous tasks, making continual learning~\citep{thrun1995lifelong,wang2024comprehensive}
central to LLM adaptation.

Recent studies on continual adaptation of LLMs have explored two
complementary mechanisms: model parameter updates
\citep{peng2024scalable,nayak2026sculpting} and task-specific soft prompts
\citep{razdaibiedina2023progressive,qiu2025continual}. Updating model
parameters provides substantial capacity for acquiring task knowledge, but
it can also interfere with prior knowledge. In contrast,
prompt-based methods preserve task-specific behaviors without modifying the
underlying model, although this limits adaptation capacity. To balance
these strengths, later methods jointly adapt model parameters and prompts
\citep{wang2024two,tiwari2026learningfastslowllms}. Yet joint adaptation
introduces an additional source of mismatch: historical prompts were
optimized for earlier model parameters, and as model parameters evolve, these prompts may no longer induce their original behaviors.

Beyond preserving prior task behaviors, continual adaptation should also enable
positive transfer across related tasks. Organizing and reusing knowledge acquired
from previous tasks should help improve performance on subsequent related tasks. Existing approaches
have explored prompt retrieval, sharing, and composition to reuse knowledge
across tasks
\citep{wang2022learning,smith2023coda,gao2024consistent,han2025semantic,hong2025rainbowprompt}. However, as the model evolves, these methods do not explicitly organize
historical task-specific prompts into a learnable structure for discovering
transferable cross-task relationships.

To address these limitations, we propose Learning an Anchored Prompt
Space (LAPS), a continual adaptation method that preserves historical task
behaviors under evolving model parameters while identifying prompts that
encourage positive transfer across tasks. LAPS contains two key elements.
First, LAPS aligns historical prompts with the evolving LLM through
on-policy self-distillation
\citep{shenfeld2026selfdistillation,zhao2026selfdistilled}. When model
parameters are updated for a new task, historical prompts are
adjusted so that they remain close to their original task behaviors under the
updated model. By restoring prior behaviors through prompting rather than
constraining the model update, this alignment leaves greater flexibility
for new task adaptation.
Second, LAPS organizes task-specific knowledge within an anchored prompt space
parameterized by a degree-$k$ B\'ezier simplex. In this construction,
aligned task-specific prompts serve as vertex anchors, and control prompts
learned from cross-task interactions determine the intermediate geometry.
Searching the resulting space allows LAPS to identify high-performing prompts
that exploit positive transfer.

Our contributions are summarized as follows:
\begin{itemize}
    \item We propose LAPS, which aligns historical prompts to preserve prior task
    behaviors without constraining model adaptation to new tasks.

    \item We introduce two-stage optimization of an anchored prompt space:
    learning control prompts from cross-task interactions and selecting
    task-specific simplex coordinates by validation performance to enable positive transfer.

    \item We demonstrate that LAPS consistently achieves a
    favorable adaptation--retention trade-off compared with
    existing distillation-based, prompt-based, and joint
    adaptation methods on TRACE~\citep{wang2023trace}
    across Qwen3 models ranging from 1.7B to 8B.
\end{itemize}

\section{Related Work}

In this section, we review continual learning literature most closely related to
LAPS: replay, knowledge distillation, and prompt-based methods. We focus on
where reusable knowledge is stored, how prior behavior is maintained, and
whether positive transfer is supported.

\noindent\textbf{Replay-based Continual Learning.}
Replay-based methods mitigate catastrophic forgetting by retaining and revisiting previous
data when learning new tasks.
Experience Replay~\citep{isele2018selective,rolnick2019experience,fedus2020revisiting}
shows that even a small memory of historical samples can help preserve
previous knowledge. Subsequent work improves replay by storing additional
training signals, such as historical logits~\citep{buzzega2020dark}, or
by learning generative models that reproduce previous task knowledge~\citep{gao2023ddgr,maekawa2023generative}. For LLM adaptation, OPR~\citep{chen2026policy} further studies on-policy replay
data for sequential updates. These methods rely on stored or regenerated historical information, while
directly mixing replay data into model updates may interfere with current task
adaptation. LAPS instead uses limited replay
data to supervise historical prompt alignment and anchored prompt space
learning, supporting retention and cross-task transfer while avoiding direct interference with new task model updates.

\noindent\textbf{Knowledge Distillation-based Continual Learning.}
Knowledge distillation mitigates forgetting by encouraging the updated model to
match the behavior of an earlier model. LwF~\citep{li2017learning} matches the earlier model's output distributions,
while later continual learning methods also distill intermediate
representations~\citep{douillard2020podnet}.
Recent work further studies self-distillation for continual learning and
restoration-style distillation for LLM behavior preservation~\citep{dong2025longred,shenfeld2026selfdistillation}. However, behavioral constraints that preserve earlier knowledge can also limit
the acquisition of new task knowledge. LAPS instead applies self-distillation to historical prompts rather than
constraining backbone updates. This leaves greater flexibility for new task
adaptation, while historical behaviors affected by model parameter updates are
recovered through the aligned prompts.

\noindent\textbf{Prompt-based Continual Learning.}
Prompt-based continual learning stores task knowledge in learnable prompts.
One common setting introduces task-specific prompts while typically keeping the
pretrained backbone frozen. L2P~\citep{wang2022learning} learns a prompt pool
for task-aware retrieval, enabling prompts to be reused across tasks.
DualPrompt~\citep{wang2022dualprompt} separates shared and task-specific prompts. ProgPrompt
~\citep{razdaibiedina2023progressive} learns a new prompt for each task and
concatenates it with previous prompts, allowing later tasks to reuse earlier
prompt knowledge. CODA-Prompt~\citep{smith2023coda} decomposes prompts into
reusable factors and learns attention-based combinations of these factors,
providing a more flexible form of prompt composition.
This prompt-only line of work reduces forgetting by storing task knowledge in
learnable prompts, but keeping the backbone fixed limits the capacity for
new task adaptation.

Another line of work combines prompt learning with model parameter updates.
ProMoT~\citep{wang2024two} first learns task-specific prompts and then updates
model parameters. Subsequent methods combine meta prompting with low-rank
parameter updates~\citep{yu2025fm} or use fast-slow adaptation mechanisms
~\citep{tiwari2026learningfastslowllms}. Jointly adapting prompts and model
parameters is difficult because historical prompts are learned under earlier
model states and may become less compatible as model parameters evolve. Existing
prompt pools, prompt composition methods, and meta prompting approaches do not
explicitly align historical prompts under changing model parameters or learn a
controlled prompt space for positive transfer. LAPS addresses this setting by aligning
historical prompts under the current model, learning B\'ezier control prompts
to shape an anchored prompt space, and selecting task-specific simplex
coordinates.

\section{Learning an Anchored Prompt Space for Continual Adaptation}
\label{sec:method}

In this section, we first formulate continual adaptation with joint
prompt-parameter updates and identify the two difficulties that motivate LAPS.
We then describe how LAPS aligns historical soft prompts, learns an anchored prompt
space, and selects the final task-specific soft prompt for each observed task.

\subsection{Problem Setup and Overview}
\label{sec:problem_setup}

We study continual adaptation over a sequence of tasks
$\mathcal{T}_{1:T}=(\mathcal{T}_1,\ldots,\mathcal{T}_T)$. Each task
$\mathcal{T}_t$ is specified by a training set of $N_t$ examples,
$\mathcal{D}_t=\{(\bm{x}_t^{(i)},\bm{y}_t^{(i)})\}_{i=1}^{N_t}$, where
$\bm{x}_t^{(i)}$ and $\bm{y}_t^{(i)}$ denote the input and target output of the
$i$-th example from task $t$. The tasks are learned sequentially, with the full training data of previous
tasks unavailable when adapting to the current task $\mathcal{T}_t$. After each learning stage, the model is evaluated on all observed
tasks. We follow the task-incremental setting~\citep{van2022three}, where the task identity is known at inference time and can therefore be used to
select the corresponding task-specific prompt.

Following recent studies~\citep{wang2024two,tiwari2026learningfastslowllms},
we consider two coupled adaptation components at stage $t$: the model
parameters $\bm{\theta}_t$ and a task-specific soft prompt. We use
$\mathbf{P}_{j,t}$ to denote the soft prompt for task $\mathcal{T}_j$ after
adaptation stage $t$, with $\mathbf{P}_{t,t}$ denoting the prompt newly learned
for the current task $\mathcal{T}_t$. In the standard joint adaptation setting, each
task-specific soft prompt is frozen after it is learned, \ie,
$\mathbf{P}_{j,t}=\mathbf{P}_{j,j}$ for all $t\geq j$. A joint
adaptation objective therefore optimizes the model parameters
$\bm{\theta}_t$ and the current task prompt $\mathbf{P}_{t,t}$ on
$\mathcal{D}_t$:
\begin{equation}
\bm{\theta}_t,\mathbf{P}_{t,t}
=
\arg\min_{\bm{\theta},\mathbf{P}}
\mathbb{E}_{(\bm{x},\bm{y})\sim\mathcal{D}_t}
\left[
\ell_t
\left(
f(\bm{x};\bm{\theta},\mathbf{P}),
\bm{y}
\right)
\right], \quad t=1,2,...,T,
\label{eq:joint_adaptation}
\end{equation}
where $f(\bm{x};\bm{\theta},\mathbf{P})$ denotes the
model prediction for input $\bm{x}$ under parameters $\bm{\theta}$ and soft
prompt $\mathbf{P}$, and $\ell_t$ denotes the task-specific supervised loss. 
Equation~(\ref{eq:joint_adaptation}) specifies an objective for adapting both
components, but does not prescribe simultaneous updates. LAPS follows the
two-stage alternating strategy of ProMoT~\citep{wang2024two}, allowing the prompt
to absorb task-specific format information before backbone adaptation,
thereby reducing the need for format-specific updates to the backbone.
Specifically, for task $\mathcal{T}_t$, LAPS first learns the current task prompt
$\mathbf{P}_{t,t}$ while keeping the previous model parameters
$\bm{\theta}_{t-1}$ fixed:
\begin{equation}
\mathbf{P}_{t,t}
=
\arg\min_{\mathbf{P}}
\mathbb{E}_{(\bm{x},\bm{y})\sim\mathcal{D}_t}
\left[
\ell_t
\left(
f(\bm{x};
\bm{\theta}_{t-1},
\mathbf{P}),
\bm{y}
\right)
\right].
\label{eq:prompt_learning}
\end{equation}
After solving Eq.~(\ref{eq:prompt_learning}), LAPS fixes
$\mathbf{P}_{t,t}$ and updates the model parameters on the current task:
\begin{equation}
\bm{\theta}_{t}
=
\arg\min_{\bm{\theta}}
\mathbb{E}_{(\bm{x},\bm{y})\sim\mathcal{D}_t}
\left[
\ell_t
\left(
f(\bm{x};
\bm{\theta},
\mathbf{P}_{t,t}),
\bm{y}
\right)
\right].
\label{eq:model_update}
\end{equation}
Although this schedule separates prompt learning from model adaptation, it does not explicitly address
knowledge retention or transfer across tasks. First, updating
model parameters from $\bm{\theta}_{t-1}$ to $\bm{\theta}_t$ for task $t$
can change the behaviors induced by historical prompts
$\{\mathbf{P}_{j,t-1}\}_{1\le j<t}$, compromising performance on earlier tasks.
Second, learning task-specific prompts separately does not explicitly capture
their transferable relationships, leaving positive transfer
unexplored. LAPS addresses these limitations through the three components
illustrated in Fig.~\ref{fig:overview}: historical prompt alignment, anchored
prompt space learning, and validation-guided prompt selection.

\begin{figure}[t]
\centering
\includegraphics[width=0.98\linewidth]{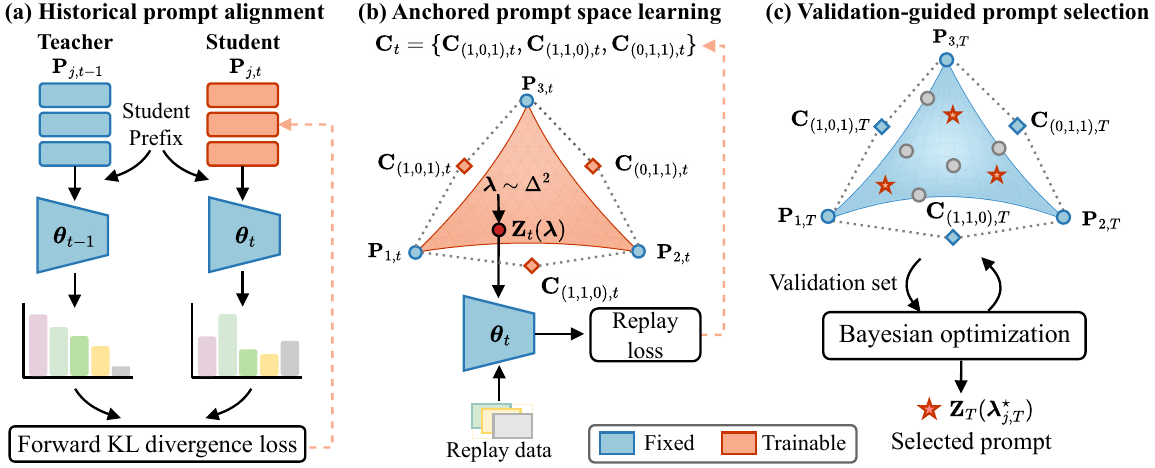}
\caption{
\textbf{Overview of LAPS.}
\textbf{(a)} After adapting to the current task $\mathcal{T}_t$, LAPS aligns
each historical prompt $\mathbf{P}_{j,t-1}$ to the updated model, producing
$\mathbf{P}_{j,t}$ through on-policy self-distillation.
\textbf{(b)} The aligned prompts $\{\mathbf{P}_{j,t}\}_{j\le t}$ are fixed as
simplex anchors, while non-vertex B\'ezier control prompts
$\mathbf{C}_{\bm{\alpha},t}$ are trained with cross-task replay losses to shape
the space between anchors.
\textbf{(c)} With the final model fixed, Bayesian optimization searches
validation scores over simplex coordinates $\bm{\lambda}$ and returns the
selected soft prompt $\mathbf{Z}_T(\bm{\lambda}_{j,T}^{\star})$ for each
observed task.
Panels (b) and (c) illustrate a quadratic B\'ezier simplex (\ie, $k=2$) over
three tasks for simplicity.
}\label{fig:overview}
\end{figure}

\subsection{Historical Prompt Alignment via Self-Distillation}

As model parameters evolve, a prompt learned at an earlier stage may no longer
induce the same behavior under the updated model. LAPS aligns
historical prompts after each model update so that they remain compatible with
the current model, without constraining the model update in
Eq.~(\ref{eq:model_update}).

Direct supervised fine-tuning of historical prompts on limited previous task replay data is insufficient because the replay set sparsely covers
previous tasks, and ground-truth targets do not capture the full behavior of
the original model-prompt pair, such as output distributions or responses to
alternative continuations. LAPS therefore uses the previous model-prompt pair
as the teacher to provide a richer behavioral target for aligning each
historical prompt with the updated model.

Prior work shows that the forward Kullback–Leibler (KL)  divergence from the original model
to the updated model, evaluated on current task data, strongly predicts forgetting
of previously learned capabilities~\citep{shenfeld2026rl}. Motivated by this
finding, we include $\mathcal{S}_t \subset \mathcal{D}_t$, a subset sampled from the current
task, alongside the stage-$t$ replay set
$\mathcal{R}_t$ when aligning historical prompts. For each task $j=1,\ldots,t-1$, the teacher is
the previous pair $(\bm{\theta}_{t-1},\mathbf{P}_{j,t-1})$, and the student is
the pair $(\bm{\theta}_t,\mathbf{P}_{j,t})$. The teacher provides a behavioral
reference for how the historical prompt functioned before the model update.
Matching its predictive distribution allows the updated prompt to recover this
behavior while keeping $\bm{\theta}_t$ fixed. LAPS aligns
$\mathbf{P}_{j,t}$ by minimizing the forward KL divergence:
\begin{equation}
\min_{\mathbf{P}_{j,t}}
\mathbb{E}_{\substack{
\bm{x}\sim\mathcal{R}_t\cup\mathcal{S}_t\\
\bm{y}_{<\tau}\sim
\pi_{\bm{\theta}_{t},\mathbf{P}_{j,t}}(\cdot\mid\bm{x})}}
\left[
D_{\mathrm{KL}}
\Bigl(
\pi_{\bm{\theta}_{t-1},\mathbf{P}_{j,t-1}}
(\cdot\mid\bm{x},\bm{y}_{<\tau})
\|
\pi_{\bm{\theta}_{t},\mathbf{P}_{j,t}}
(\cdot\mid\bm{x},\bm{y}_{<\tau})
\Bigr)
\right],
\label{eq:historical_alignment}
\end{equation}
where $\bm{y}_{<\tau}$ is sampled
on-policy from the student and used as the shared prefix for both teacher and
student distributions, and $\pi_{\bm{\theta},\mathbf{P}}(\cdot\mid\bm{x},\bm{y}_{<\tau})$
denotes the next-token predictive distribution induced by parameters
$\bm{\theta}$ and prompt $\mathbf{P}$. Optimizing Eq.~(\ref{eq:historical_alignment}) realigns earlier prompts to
the updated model, so new task adaptation need not preserve every previous
behavior in the model parameters.

\subsection{Anchored Prompt Space Learning}

After alignment, LAPS has a set of task prompts
$\{\mathbf{P}_{1,t},\ldots,\mathbf{P}_{t,t}\}$ that remain usable under the
current model $\bm{\theta}_t$. LAPS treats these task-specific soft prompts as anchors for their corresponding
task knowledge. Using the anchors alone restricts inference to independently learned task solutions,
leaving potential cross-task transfer unexplored. A natural alternative is convex interpolation, whose mixing
weight can be optimized to exploit cross-task relationships. However, the
resulting prompts remain restricted to linear combinations of the fixed anchors.
LAPS introduces learnable non-vertex B\'ezier control prompts to facilitate nonlinear
interpolation and shape the prompt space through cross-task supervision.

At stage $t$, LAPS defines a simplex over the observed tasks:
\begin{equation}
\Delta^{t-1}
=
\left\{
\bm{\lambda}\in\mathbb{R}_{\geq0}^{t}
\;\middle|\;
\sum_{j=1}^{t}\lambda_j=1
\right\}.
\end{equation}
Each component $\lambda_j$ controls the relative emphasis on task
$\mathcal{T}_j$. For any coordinate
$\bm{\lambda}\in\Delta^{t-1}$, LAPS constructs a prompt
$\mathbf{Z}_t(\bm{\lambda})$ with a degree-$k$ B\'ezier simplex:
\begin{equation}
\mathbf{Z}_{t}(\bm{\lambda})
=
\sum\nolimits_{j=1}^{t}
\lambda_j^{k}\mathbf{P}_{j,t}
+
\sum\nolimits_{\bm{\alpha}}
\frac{k!}{\prod_{j=1}^{t}\alpha_j!}
\prod_{j=1}^{t}\lambda_j^{\alpha_j}
\mathbf{C}_{\bm{\alpha},t}.
\label{eq:anchored_bezier}
\end{equation}
As shown in Fig.~\ref{fig:overview}(b), the task prompts
$\mathbf{P}_{j,t}$ form the simplex vertices, while the non-vertex control
prompts $\mathbf{C}_{\bm{\alpha},t}$ shape the space between them. In the
second sum, $\bm{\alpha}=(\alpha_1,\ldots,\alpha_t)$ ranges over the
non-vertex degree-$k$ multi-indices: each $\alpha_j\in\{0,1,\ldots,k\}$,
$\sum_{j=1}^{t}\alpha_j=k$, and
$\bm{\alpha}\notin\{k\bm{e}_1,\ldots,k\bm{e}_t\}$, where
$\bm{e}_j$ is the $j$-th unit vector. Each 
$\bm{\alpha}$ indexes one learnable control prompt
$\mathbf{C}_{\bm{\alpha},t}$, so LAPS adds
$\binom{t+k-1}{k}-t$ control prompts at stage $t$. For coordinate
$\bm{\lambda}$, the Bernstein weight
$\frac{k!}{\prod_{j=1}^{t}\alpha_j!}
\prod_{j=1}^{t}\lambda_j^{\alpha_j}$
determines how strongly $\mathbf{C}_{\bm{\alpha},t}$ contributes to
$\mathbf{Z}_t(\bm{\lambda})$.
When $k=1$, no non-vertex control prompts exist, and
Eq.~(\ref{eq:anchored_bezier}) reduces to convex interpolation among task
anchors. 
Notably, these additional control prompts account for less
than $0.5\%$ of the backbone parameters.

During anchored prompt space learning, LAPS keeps the model parameters $\bm{\theta}_t$ and the task
anchors fixed and optimizes only the non-vertex control prompts, collected as
$\mathbf{C}_t$. 
For each observed task $j\leq t$, the loss induced by the coordinate
$\bm{\lambda}$ and control prompts $\mathbf{C}_t$ is evaluated on examples from the stage-$t$ replay set $\mathcal{R}_t$:
\begin{equation}
\ell_{j,t}(\bm{\lambda},\mathbf{C}_t)
=
\mathbb{E}_{(\bm{x},\bm{y})\sim\mathcal{R}_t}
\left[
\ell_j
\left(
f(\bm{x};
\bm{\theta}_t,
\mathbf{Z}_t(\bm{\lambda})),
\bm{y}
\right)
\right].
\end{equation}
We use the same vector $\bm{\lambda}$ both to specify a prompt's location in
the simplex and to weight the task losses, coupling the generated prompt with
its intended task preferences. We then aggregate the weighted
losses using a smooth Tchebycheff scalarization, which smoothly emphasizes the
largest weighted loss:
\begin{equation}
\min_{\mathbf{C}_t}
\mathbb{E}_{\bm{\lambda}\sim\Delta^{t-1}}
\left[
\mu
\log
\sum\nolimits_{j=1}^{t}
\exp
\left(
\frac{\lambda_j
\ell_{j,t}(\bm{\lambda},\mathbf{C}_t)}
{\mu}
\right)
\right],
\label{eq:control_scalarization}
\end{equation}
where $\mu>0$ controls the smoothness of the scalarization. If task losses differ substantially in magnitude,
they can be normalized before scalarization. Optimizing
Eq.~(\ref{eq:control_scalarization}) learns the non-vertex control prompts
across sampled trade-offs $\bm{\lambda}$, allowing cross-task interactions to
shape the anchored prompt space for subsequent task-specific prompt selection.

\subsection{Validation-Guided Prompt Selection}

After anchored prompt space learning, LAPS fixes the final model
$\bm{\theta}_T$, task anchors, and B\'ezier control prompts, and uses
validation data to select a soft prompt from the learned anchored space for
each observed task. For task $\mathcal{T}_j$, the prompt
$\mathbf{Z}_T(\bm{\lambda})$, parameterized by the simplex coordinate
$\bm{\lambda}\in\Delta^{T-1}$, is optimized with the fixed model on the
task's validation set:
\begin{equation}
\bm{\lambda}_{j,T}^{\star}
=
\argmax_{\bm{\lambda}\in\Delta^{T-1}}
\mathrm{Val}_{j,T}(\bm{\lambda}),
\qquad j=1,\ldots,T,
\label{eq:task_lambda_selection}
\end{equation}
where $\mathrm{Val}_{j,T}(\bm{\lambda})$ denotes the resulting task-specific
validation score, which may be discrete or otherwise unsuitable for
gradient-based optimization (\eg, accuracy, F1 score, or exact match). LAPS therefore treats
$\mathrm{Val}_{j,T}$ as a black-box objective and approximates
Eq.~(\ref{eq:task_lambda_selection}) with Bayesian optimization~\citep{snoek2012practical,snoek2014input}. For each task, we fit a Gaussian process surrogate to the observed
coordinate–score pairs and select new coordinates using Expected
Improvement~\citep{jones1998efficient}, balancing predicted validation
performance against uncertainty. 
The leading candidate coordinates are
reevaluated on the full validation set over repeated runs, with the task's own
anchor added if it is not already among the finalists. The coordinate with the
highest mean validation score is used to construct the final soft prompt.
Appendix~\ref{app:bo_details} provides further details on the search and evaluation.

\section{Experiment}
In this section, we evaluate LAPS on continual instruction tuning. We first give the
TRACE setup and comparison methods, then analyze aggregate performance,
task-wise adaptation dynamics, and the learned prompt space. Finally, we present
ablations and diagnostic studies of LAPS components.

\subsection{Experimental Setups}

\noindent\textbf{Benchmark.}
We conduct experiments on TRACE~\citep{wang2023trace}, an eight-task
continual instruction tuning benchmark consisting of C-STANCE, FOMC,
MeetingBank, Py150, ScienceQA, NumGLUE-cm, NumGLUE-ds, and 20Minuten.
Following prior work~\citep{chen2026policy}, we subsample 5,000 training examples from
each task and use the official evaluation set.

\noindent\textbf{Training.}
We build all methods on Qwen3-1.7B, Qwen3-4B, and Qwen3-8B \citep{yang2025qwen3}.
For all methods that update model parameters, we perform full-parameter
fine-tuning using AdamW with a learning rate of $1\times10^{-5}$, a global batch size of $128$, a maximum sequence length of $2,048$ tokens,
and bfloat16 precision. Unless otherwise stated, prompt-based methods use $32$ prompt tokens. Following the task order in TRACE, we train the eight
tasks for $[5,3,7,5,3,5,5,7]$ epochs, respectively. For replay-based methods,
the replay ratio is set to $1\%$. For LAPS, we use a cubic B\'ezier simplex
with degree $k=3$, resulting in 112 control prompts for eight tasks. The
scalarization smoothing parameter is $\mu=0.1$. We construct
the alignment dataset by combining replay data from previous tasks and data
from the current task at a $1\!:\!1$ ratio. Bayesian optimization is performed
on the official validation set with a budget of 100 evaluations per task,
except for MeetingBank, where we use 50 evaluations.

\noindent\textbf{Baselines.}
We compare against baselines that stress different continual-adaptation
mechanisms. SFT is a \textit{parameter-only adaptation} baseline; vanilla
replay~\citep{rolnick2019experience} and on-policy replay
(OPR-SC and OPR-RU, \citealp{chen2026policy}) add previous examples; LwF
~\citep{li2017learning} uses \textit{distillation-based adaptation};
ProgPrompt~\citep{razdaibiedina2023progressive} performs
\textit{prompt-only adaptation}; and ProMoT~\citep{wang2024two} implements
\textit{joint prompt-parameter adaptation}. Together, these baselines enable a systematic comparison of knowledge retention across parameter-only adaptation, frozen-backbone prompt tuning, and joint prompt–parameter adaptation, including the use of replay or distillation to mitigate forgetting.

\noindent\textbf{Metrics.}
We report average performance (\textbf{Average}) across all tasks and backward
transfer (\textbf{BWT}, \citealp{lopez2017gradient}) as a measure of forgetting.
Let $R_{j,t}$ denote the performance on task $\mathcal{T}_j$ after
learning stage $t$. We define
$\mathrm{BWT}=\frac{1}{T-1}\sum_{j=1}^{T-1}(R_{j,T}-R_{j,j})$.
Higher Average indicates better final performance averaged across tasks, while
higher BWT indicates stronger retention of earlier tasks. We use stochastic decoding
with a temperature of $0.1$ and report mean scores over eight inference runs
to reduce variability.

We summarize method characteristics and trainable parameter counts in
Appendix~\ref{chara_para_count}, detail dataset statistics and evaluation metrics
in Appendix~\ref{app:dataset_details}, and present additional experimental
analyses in Appendix~\ref{additionalexp}.

\begin{table*}[t]
    \centering
\caption{
\textbf{Performance across Qwen3 model scales.}
The best and second best results at each model scale
are shown in \textbf{bold} and \underline{underlined}, respectively.
}
\label{tab:main_results_all_scales}

    \scriptsize
    \setlength{\tabcolsep}{3pt}
    \renewcommand{\arraystretch}{1.0}

    \resizebox{\textwidth}{!}{%
    \begin{tabular}{lcccccccc|cc}
        \toprule
        \textbf{Method}
        & \textbf{C-STANCE}
        & \textbf{FOMC}
        & \textbf{MeetingBank}
        & \textbf{Py150}
        & \textbf{ScienceQA}
        & \textbf{NG-cm}
        & \textbf{NG-ds}
        & \textbf{20Minuten}
        & \textbf{Average} $\uparrow$
        & \textbf{BWT} $\uparrow$ \\
        \midrule

        \rowcolor{gray!15}
        \multicolumn{11}{l}{\textbf{Qwen3-1.7B}} \\
        Initial
        & $36.3$ & $32.7$ & $18.4$ & $2.5$ & $63.4$
        & $70.2$ & $30.2$ & $35.8$ & $36.2$ & $\text{N/A}$ \\
        SFT
        & $\underline{53.1}$ & $58.5$ & $20.1$ & $52.6$ & $76.7$
        & $46.3$ & $53.8$ & $38.8$ & $50.0$ & $-9.3$ \\

        Replay
        & $52.8$
        & $62.4$
        & $46.6$
        & $56.6$
        & $73.4$
        & $51.1$
        & $53.2$
        & $38.9$
        & $54.4$
        & $-4.1$ \\

        OPR-RU
        & $50.4$
        & $58.2$
        & $47.0$
        & $54.7$
        & $77.1$
        & $\underline{55.7}$
        & $51.7$
        & $38.8$
        & $54.2$
        & $-4.9$ \\

        OPR-SC
        & $50.7$
        & $51.9$
        & $38.6$
        & $55.0$
        & $77.7$
        & $47.5$
        & $52.6$
        & $38.7$
        & $51.6$
        & $-7.6$ \\
                LwF
        & $45.9$
        & $28.4$
        & $38.2$
        & $43.4$
        & $41.6$
        & $43.1$
        & $53.5$
        & $38.0$
        & $41.5$
        & $-13.8$ \\

        ProgPrompt
        & $\mathbf{53.5}$
        & $60.6$
        & $47.7$
        & $49.8$
        & $73.8$
        & $45.4$
        & $47.6$
        & $38.4$
        & $52.1$
        & $\mathbf{0.0}$ \\

        ProMoT
        & $50.9$
        & $\underline{62.7}$
        & $\underline{50.9}$
        & $\underline{59.2}$
        & $\mathbf{84.2}$
        & $47.7$
        & $\mathbf{67.7}$
        & $\mathbf{40.5}$
        & $\underline{58.0}$
        & $-4.4$ \\
        \noalign{\vskip 1pt}
        \hline
        \noalign{\vskip 1pt}
        \textbf{LAPS (Ours)}
        & $52.0$
        & $\mathbf{69.4}$
        & $\mathbf{58.7}$
        & $\mathbf{60.0}$
        & $\underline{83.4}$
        & $\mathbf{56.0}$
        & $\underline{64.8}$
        & $\underline{40.0}$
        & $\mathbf{60.5}$
        & $\underline{-1.2}$ \\

        \midrule
        \rowcolor{gray!15}
        \multicolumn{11}{l}{\textbf{Qwen3-4B}} \\
        Initial
        & $42.7$ & $5.0$ & $20.6$ & $2.4$ & $64.6$
        & $76.4$ & $38.2$ & $35.9$ & $35.7$ & $\text{N/A}$ \\
        SFT
        & $52.5$ & $66.2$ & $54.3$ & $37.7$ & $89.2$
        & $\mathbf{66.5}$ & $\underline{77.3}$ & $40.8$ & $60.6$ & $-6.7$ \\

        Replay
        & $55.3$
        & $71.4$
        & $59.0$
        & $\underline{62.3}$
        & $91.9$
        & $64.0$
        & $75.0$
        & $40.5$
        & $\underline{64.9}$
        & $-2.7$ \\

        OPR-RU
        & $53.0$
        & $\mathbf{73.3}$
        & $57.7$
        & $62.2$
        & $91.1$
        & $63.9$
        & $76.9$
        & $\mathbf{41.1}$
        & $\underline{64.9}$
        & $-1.2$ \\

        OPR-SC
        & $52.8$
        & $61.9$
        & $58.0$
        & $60.4$
        & $82.7$
        & $62.8$
        & $\mathbf{77.8}$
        & $40.4$
        & $62.1$
        & $-5.5$ \\
                LwF
        & $56.0$
        & $71.1$
        & $47.7$
        & $38.8$
        & $83.3$
        & $65.1$
        & $69.4$
        & $39.9$
        & $58.9$
        & $\mathbf{+0.2}$ \\

        ProgPrompt
        & $\underline{57.8}$
        & $62.2$
        & $51.4$
        & $57.9$
        & $79.7$
        & $55.4$
        & $52.4$
        & $39.0$
        & $57.0$
        & $\underline{0.0}$ \\

        ProMoT
        & $51.9$
        & $59.3$
        & $\mathbf{62.9}$
        & $\mathbf{63.0}$
        & $\mathbf{93.4}$
        & $63.1$
        & $72.7$
        & $\underline{41.0}$
        & $63.4$
        & $-2.0$ \\
        \noalign{\vskip 1pt}
        \hline
        \noalign{\vskip 1pt}
        \textbf{LAPS (Ours)}
        & $\mathbf{59.5}$
        & $\underline{71.9}$
        & $\underline{60.3}$
        & ${61.6}$
        & $\underline{92.4}$
        & $\underline{66.4}$
        & $76.9$
        & $40.6$
        & $\mathbf{66.2}$
        & $-0.3$ \\

        \midrule
        \rowcolor{gray!15}
        \multicolumn{11}{l}{\textbf{Qwen3-8B}} \\
        Initial
        & $52.6$ & $21.0$ & $20.8$ & $2.4$ & $68.0$
        & $77.8$ & $38.4$ & $36.4$ & $39.6$ & $\text{N/A}$ \\
        SFT
        & $47.3$ & $63.5$ & $50.0$ & $50.0$ & $73.4$
        & $67.9$ & $77.7$ & $41.5$ & $58.9$ & $-10.1$ \\

        Replay
        & $52.3$
        & $71.9$
        & $55.2$
        & $62.1$
        & $92.4$
        & $67.1$
        & $\mathbf{79.5}$
        & $41.4$
        & $65.2$
        & $-2.9$ \\

        OPR-RU
        & $52.6$
        & $\mathbf{74.5}$
        & $60.8$
        & $\underline{62.6}$
        & $92.0$
        & $69.1$
        & $\underline{78.5}$
        & $41.1$
        & $66.4$
        & $-2.7$ \\

        OPR-SC
        & $51.8$
        & $72.6$
        & $57.0$
        & $62.1$
        & $\underline{93.0}$
        & $\mathbf{69.8}$
        & $76.7$
        & $41.5$
        & $65.6$
        & $-2.5$ \\
                LwF
        & $49.5$
        & $72.0$
        & $46.3$
        & $30.8$
        & $84.7$
        & $52.0$
        & $66.1$
        & $39.6$
        & $55.1$
        & $\mathbf{+4.3}$ \\

        ProgPrompt
        & $\mathbf{58.6}$
        & ${64.9}$
        & $54.9$
        & $60.7$
        & $84.9$
        & $47.5$
        & $56.0$
        & $39.3$
        & $58.4$
        & $0.0$ \\

        ProMoT
        & $55.9$
        & $69.8$
        & $\underline{64.5}$
        & $\mathbf{64.1}$
        & $\mathbf{94.6}$
        & $\underline{69.3}$
        & $\underline{78.5}$
        & $\underline{41.6}$
        & $\underline{67.3}$
        & $-0.2$  \\
        \noalign{\vskip 1pt}
        \hline
        \noalign{\vskip 1pt}
        \textbf{LAPS (Ours)}
        & $\underline{57.6}$
        & $\underline{73.4}$
        & $\mathbf{66.4}$
        & $\mathbf{64.1}$
        & $92.4$
        & $66.2$
        & $78.3$
        & $\mathbf{41.8}$
        & $\mathbf{67.5}$
        & $\underline{+1.2}$ \\

        \bottomrule
    \end{tabular}%
    }
\end{table*}

\subsection{Main Results}

Table~\ref{tab:main_results_all_scales} shows the adaptation--retention
pattern across the three Qwen3 model scales. Parameter updating methods such
as SFT, Replay, OPR-SC, OPR-RU, and ProMoT often improve new task performance,
but their BWT values show that earlier task behaviors can still be overwritten.
ProgPrompt has the opposite profile: freezing the backbone helps preserve performance on previous tasks, yet its Average remains lower than the best parameter-updating
methods on several scales. Among these baselines, ProMoT provides the most direct comparison with LAPS
because both combine soft prompts with model parameter updates. LAPS achieves
the highest Average across all three model scales---60.5\%, 66.2\%, and 67.5\%
on Qwen3-1.7B, Qwen3-4B, and Qwen3-8B, respectively---while maintaining BWT
close to or above zero. Compared with ProMoT, LAPS improves Average by
$0.2\%\text{--}2.8\%$ and BWT by $1.4\%\text{--}3.2\%$.
These gains across model scales indicate a consistently more favorable
adaptation--retention trade-off.

\noindent\textbf{Continual Adaptation Dynamics.} Beyond the aggregate table, the task-wise trajectories in
Fig.~\ref{fig:retention1_7} reveal when forgetting occurs during
the continual sequence. On Qwen3-1.7B, SFT illustrates the failure mode of
sequential fine-tuning: later updates are followed by
degradation on several earlier tasks. ProgPrompt keeps the backbone and
previous prompts fixed, which helps retention, but this constraint leaves less
room for new task adaptation.
ProMoT adapts both prompts and model parameters and improves performance on several new tasks,
yet earlier tasks such as Py150 and NumGLUE-cm still decline as the
model parameters continue to change across the task sequence.

In contrast, LAPS exhibits a different trajectory. After sequential adaptation ends at $\mathcal{T}_8$, indicated by the vertical
gray dashed line, the model is fixed and Bayesian optimization selects a prompt
for each task from the learned anchored prompt space. The rightmost point in
each subplot shows performance with the selected prompt, compared with the
preceding pre-search point. Improvements on most tasks suggest that useful
task-specific regions exist beyond the original task anchors.
\begin{figure}[!t]
    \centering
\includegraphics[width=\linewidth]{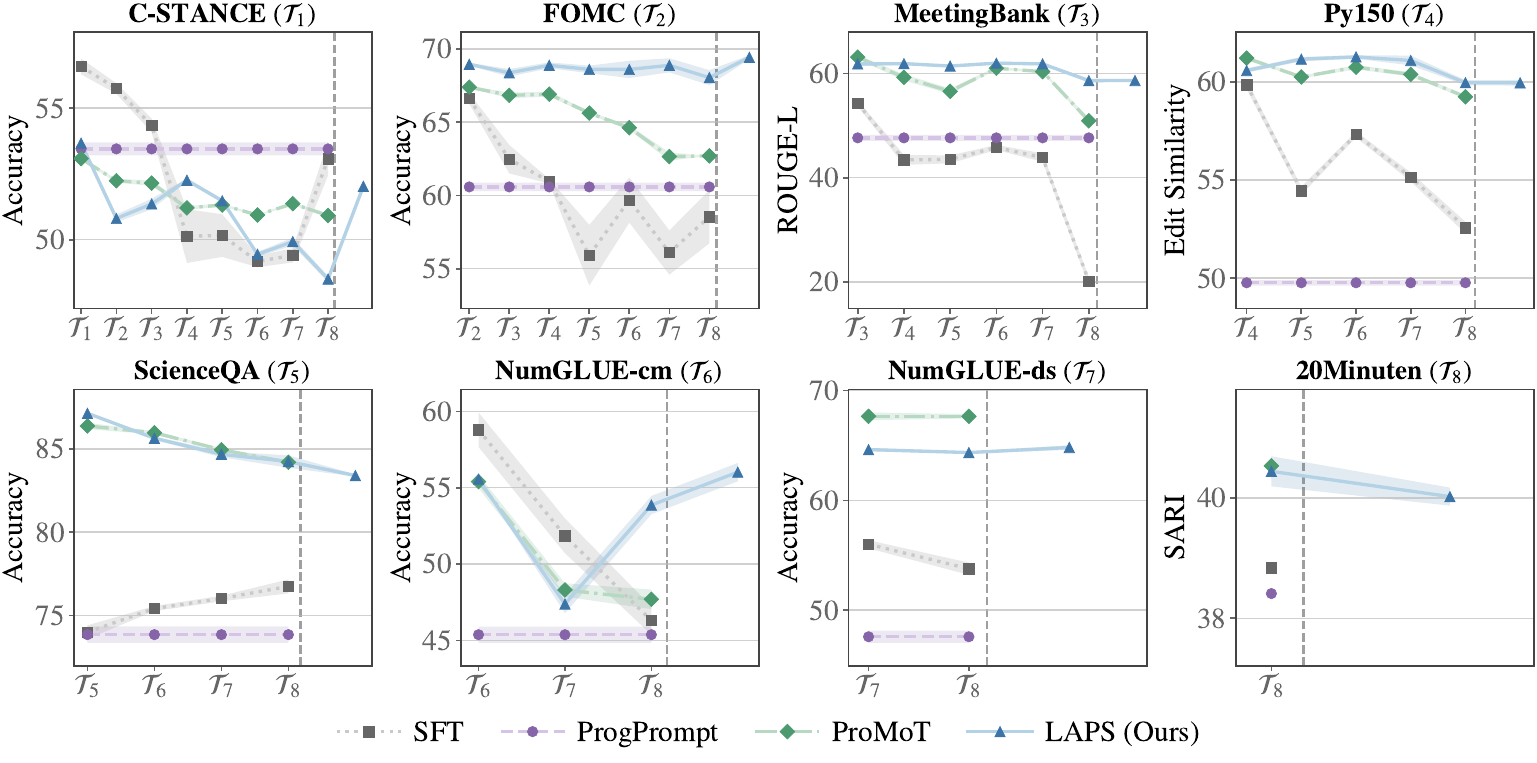}
    \caption{
\textbf{Task-wise continual adaptation trajectories using Qwen3-1.7B.}
Each panel tracks a task's performance from its initial learning stage through
subsequent adaptation stages. For LAPS, these points use the task's
stage-specific anchor prompt $\mathbf{P}_{j,t}$. The gray dashed line separates
these sequential evaluations from the final evaluation using the
validation-selected prompt $\mathbf{Z}_T(\bm{\lambda}_{j,T}^{\star})$.
Performance drops during continual adaptation indicate forgetting as new tasks
are learned sequentially.
}
    \label{fig:retention1_7}
\end{figure}
\begin{figure*}[t]
    \centering
        \begin{minipage}[t]{0.535\textwidth}
        \centering
        \includegraphics[width=\linewidth]{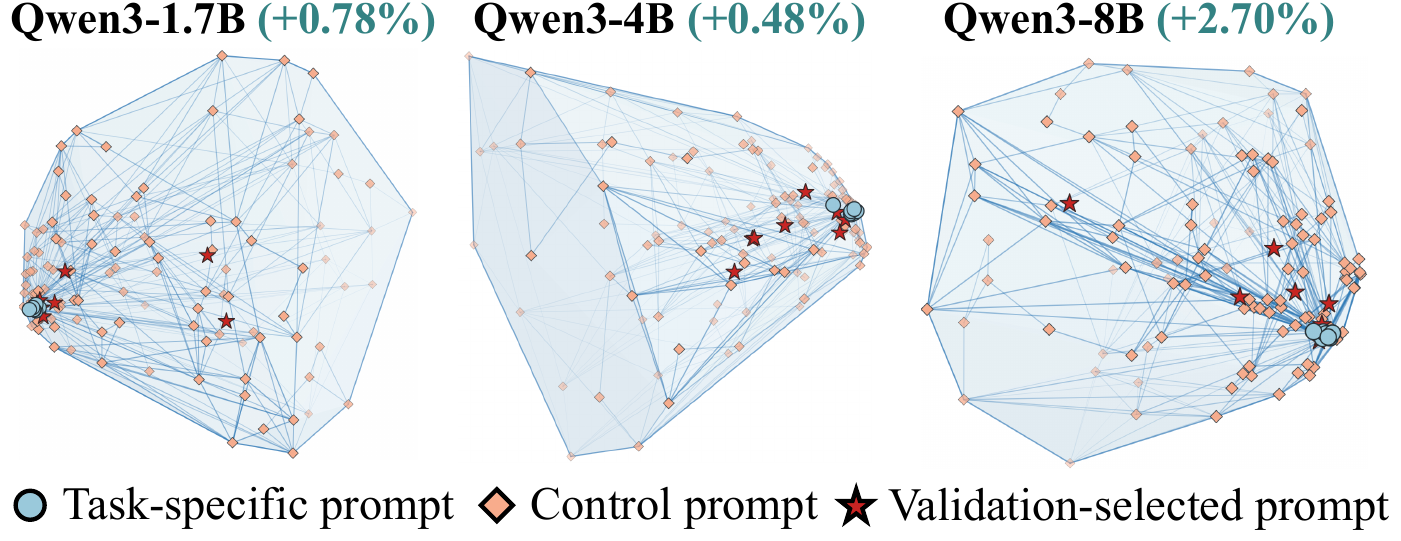}
        \caption{PCA projection of learned anchored prompt spaces across Qwen3 model scales.}
        \label{fig:manifold}
    \end{minipage}
    \hfill
    \begin{minipage}[t]{0.44\textwidth}
        \centering
        \includegraphics[width=\linewidth]{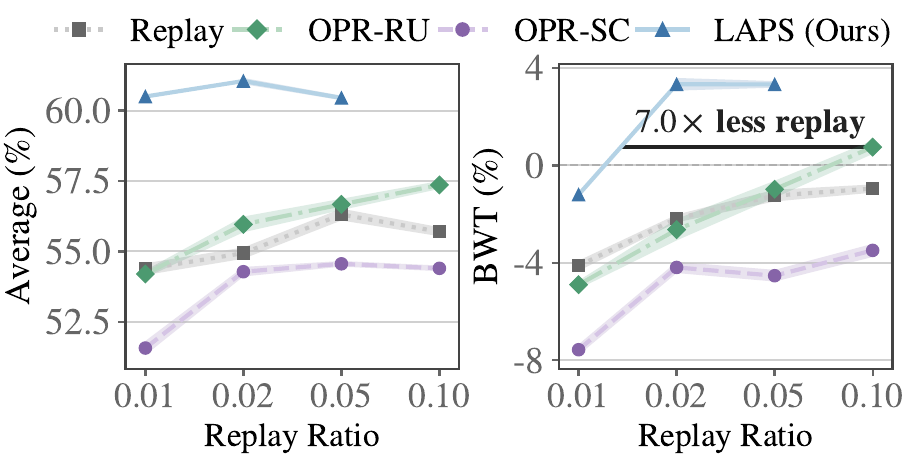}
\caption{Average performance and BWT across replay ratios on Qwen3-1.7B.}
        \label{fig:replay_ratio}
    \end{minipage}
\end{figure*}

\noindent\textbf{Visualization of the Learned Prompt Space.}
Figure~\ref{fig:manifold} uses principal component analysis (PCA) to examine where validation-selected prompts
lie relative to the original per-task prompt anchors. The selected prompts
are generally not located at the anchors and improve average performance by
$0.78\%$, $0.48\%$, and $2.70\%$ over using the anchors on Qwen3-1.7B,
Qwen3-4B, and Qwen3-8B, respectively.
This pattern suggests that the B\'ezier control prompts define a searchable
prompt space in which useful task-specific prompts need not coincide with the
original anchors. Qualitative examples in
Figs.~\ref{fig:qualitative_meetingbank} and~\ref{fig:qualitative_fomc}
illustrate how the selected prompts improve summary conciseness and
classification correctness, respectively.

\subsection{Ablations and Diagnostic Analyses}

\noindent\textbf{Effect of LAPS Components.}
Table~\ref{tab:ablation_alignment_space_selection} examines the contributions
of historical prompt alignment, anchored prompt space learning, and
validation-guided prompt selection. The full model outperforms all ablated variants,
indicating that the gains do not arise from any
single component alone. Alignment has the clearest effect on retention: removing it causes the
largest degradation in BWT while also lowering Average. This pattern is
consistent with the role of alignment, which keeps historical prompts
compatible with the updated model for each new task.

Anchored prompt space learning and validation-guided prompt selection provide additional benefits once alignment is
present. Removing either component weakens the final trade-off. These ablations
examine two distinct contributions: learning useful intermediate prompts through
control prompt optimization, and selecting an effective prompt for each task
from the learned space. The results suggest complementary roles: alignment preserves
historical behavior, while prompt space learning and selection support positive
transfer across different tasks.

\noindent\textbf{Transferability of LAPS Components.}
Table~\ref{tab:prompt_methods_with_laps} compares replay augmentation with
integrating the anchored prompt space learning and validation-guided prompt selection of LAPS into
ProgPrompt and ProMoT. Adding replay alone provides limited or inconsistent benefits and can
even reduce performance. Adding prompt space learning and selection improves both
Average and BWT across the reported model scales. 
These comparisons suggest that organizing and searching the prompt space
provides additional adaptation and retention benefits beyond replay alone.

\begin{table*}[t!]
    \centering
    \begin{minipage}[t]{0.52\textwidth}
        \centering
        \caption{
\textbf{Ablation of the three LAPS components:} historical prompt alignment, anchored
prompt space learning, and validation-guided prompt selection.
}
        \label{tab:ablation_alignment_space_selection}

        \scriptsize
        \setlength{\tabcolsep}{2pt}
        \renewcommand{\arraystretch}{1.4}

        \resizebox{\linewidth}{!}{%
        \begin{tabular}{ccccccc}
            \toprule
            \multicolumn{3}{c}{\textbf{Component}}
            & \multicolumn{2}{c}{\textbf{Qwen3-1.7B}}
            & \multicolumn{2}{c}{\textbf{Qwen3-4B}} \\
            \cmidrule(lr){1-3}
            \cmidrule(lr){4-5}
            \cmidrule(lr){6-7}

            \textbf{Alignment} & \textbf{Space Learning} & \textbf{Selection}
            & \textbf{Average} & \textbf{BWT}
            & \textbf{Average} & \textbf{BWT} \\
            \midrule

            \cmark & \cmark & \cmark
            & $\mathbf{60.5}$ & $\mathbf{-1.2}$
            & $\mathbf{66.2}$ & $\mathbf{-0.3}$ \\

            \xmark & \cmark & \cmark
            & $58.8$ & $-4.0$
            & $65.6$ & $-0.4$ \\

            \cmark & \xmark & \cmark
            & $59.6$ & $-2.3$
            & $65.5$ & $-0.5$ \\

            \cmark & \cmark & \xmark
            & $59.8$ & $-2.1$
            & $65.8$ & $-0.7$ \\

            \xmark & \xmark & \xmark
            & $57.3$ & $-5.8$
            & $64.3$ & $-1.9$ \\

            \bottomrule
        \end{tabular}%
        }
    \end{minipage}
    \hfill
 \begin{minipage}[t]{0.47\textwidth}
    \centering
    \captionof{table}{
\textbf{Prompt-based continual adaptation methods with replay and LAPS components.}
$\dagger$
denotes the variant without prompt concatenation.
}
    \label{tab:prompt_methods_with_laps}

    \scriptsize
    \setlength{\tabcolsep}{2pt}
    \renewcommand{\arraystretch}{1.125}

    \resizebox{\linewidth}{!}{%
    \begin{tabular}{lcccc}
        \toprule
        \multirow{2}{*}{\textbf{Method}}
        & \multicolumn{2}{c}{\textbf{Qwen3-1.7B}}
        & \multicolumn{2}{c}{\textbf{Qwen3-4B}} \\
        \cmidrule(lr){2-3}
        \cmidrule(lr){4-5}

        & \textbf{Average} & \textbf{BWT}
        & \textbf{Average} & \textbf{BWT} \\
        \midrule

        ProgPrompt$^{\dagger}$
        & $52.3$ & $0.0$
        & $58.9$ & $0.0$ \\

        \quad + Replay
        & $50.8$ & $0.0$
        & $55.1$ & $0.0$ \\

        \quad + Space Learning \& Selection
        & $\mathbf{53.2}$ & $\mathbf{+1.0}$
        & $\mathbf{60.5}$ & $\mathbf{+1.8}$ \\

        \midrule

        ProMoT
        & $58.0$ & $-4.4$
        & $63.4$ & $-2.0$ \\

        \quad + Replay
        & $54.2$ & $-6.9$
        & $63.9$ & $-2.8$ \\

        \quad + Space Learning \& Selection
        & $\mathbf{58.8}$ & $\mathbf{-4.0}$
        & $\mathbf{65.6}$ & $\mathbf{-0.4}$ \\

        \bottomrule
    \end{tabular}%
    }
\end{minipage}
\end{table*}

\begin{table*}[t]
    \centering
    \caption{\textbf{Effect of the soft prompt length.}
    Results are averaged over eight runs.}
    \label{tab:soft_prompt_number}

    \small
    \setlength{\tabcolsep}{6pt}
    \renewcommand{\arraystretch}{1.12}

    \resizebox{\textwidth}{!}{%
    \begin{tabular}{lcccccccccccc}
        \toprule
        \multirow{3}{*}{\textbf{Method}}
        & \multicolumn{6}{c}{\textbf{Qwen3-1.7B}}
        & \multicolumn{6}{c}{\textbf{Qwen3-4B}} \\
        \cmidrule(lr){2-7}
        \cmidrule(lr){8-13}

        & \multicolumn{2}{c}{\textbf{Length = 32}}
        & \multicolumn{2}{c}{\textbf{Length = 64}}
        & \multicolumn{2}{c}{\textbf{Length = 128}}
        & \multicolumn{2}{c}{\textbf{Length = 32}}
        & \multicolumn{2}{c}{\textbf{Length = 64}}
        & \multicolumn{2}{c}{\textbf{Length = 128}} \\
        \cmidrule(lr){2-3}
        \cmidrule(lr){4-5}
        \cmidrule(lr){6-7}
        \cmidrule(lr){8-9}
        \cmidrule(lr){10-11}
        \cmidrule(lr){12-13}

        & \textbf{Average} & \textbf{BWT}
        & \textbf{Average} & \textbf{BWT}
        & \textbf{Average} & \textbf{BWT}
        & \textbf{Average} & \textbf{BWT}
        & \textbf{Average} & \textbf{BWT}
        & \textbf{Average} & \textbf{BWT} \\
        \midrule

        ProgPrompt
        & $52.1$ & $\mathbf{0.0}$
        & $49.2$ & $\mathbf{0.0}$
        & $51.2$ & $0.0$
        & $57.0$ & $\mathbf{0.0}$
        & $57.4$ & $\mathbf{0.0}$
        & $57.5$ & $\mathbf{0.0}$ \\

        ProMoT
        & $58.0$ & $-4.4$
        & $55.8$ & $-4.6$
        & $48.6$ & $-6.9$
        & $63.4$ & $-2.0$
        & ${65.5}$ & $-0.2$
        & ${62.6}$ & $-3.2$ \\
\midrule
        \textbf{LAPS (Ours)}
        & $\mathbf{60.5}$ & $-1.2$
        & $\mathbf{59.5}$ & $-1.6$
        & $\mathbf{58.5}$ & $\mathbf{0.7}$
        & $\mathbf{66.2}$ & $-0.3$
        & $\mathbf{65.6}$ & $-0.4$
        & $\mathbf{64.7}$ & ${-1.3}$ \\

        \bottomrule
    \end{tabular}%
    }
\end{table*}
\begin{table*}[t!]
    \centering

    \begin{minipage}[t]{0.51\textwidth}
    \centering
    \captionof{table}{Effect of the B\'ezier degree $k$.}
    \label{tab:bezier_degree}

    \small
    \setlength{\tabcolsep}{6pt}
    \renewcommand{\arraystretch}{1.06}

    \resizebox{\linewidth}{!}{%
    \begin{tabular}{cccccc}
        \toprule
        \multirow{2}{*}{$k$}
        & \multirow{2}{*}{\textbf{\# Control Prompts}}
        & \multicolumn{2}{c}{\textbf{Qwen3-1.7B}}
        & \multicolumn{2}{c}{\textbf{Qwen3-4B}} \\
        \cmidrule(lr){3-4}
        \cmidrule(lr){5-6}

        & & \textbf{Average} & \textbf{BWT}
        & \textbf{Average} & \textbf{BWT} \\
        \midrule

        $1$ & $0$
        & $59.6$ & $-2.3$
        & $65.5$ & $-0.5$ \\

        $2$ & $28$
        & $60.4$ & $-2.1$
        & $65.7$ & $-0.4$ \\

        $3$ & $112$
        & $\mathbf{60.5}$ & $\mathbf{-1.2}$
        & $\mathbf{66.2}$ & $\mathbf{-0.3}$ \\

        $4$ & $322$
        & $59.3$ & $-2.4$
        & $65.2$  & $-0.9$ \\

        \bottomrule
    \end{tabular}%
    }
\end{minipage}
    \hfill
    \begin{minipage}[t]{0.48\textwidth}
        \centering
        \captionof{table}{Compatibility of LAPS with on-policy replay.}
        \label{tab:LAPS_with_onpolicy_methods}

        \small
        \setlength{\tabcolsep}{10pt}
        \renewcommand{\arraystretch}{1}

        \resizebox{\linewidth}{!}{%
        \begin{tabular}{lcccc}
            \toprule
            \multirow{2}{*}{\textbf{Method}}
            & \multicolumn{2}{c}{\textbf{Qwen3-1.7B}}
            & \multicolumn{2}{c}{\textbf{Qwen3-4B}} \\
            \cmidrule(lr){2-3}
            \cmidrule(lr){4-5}

            & \textbf{Average} & \textbf{BWT}
            & \textbf{Average} & \textbf{BWT} \\
            \midrule

            OPR-SC
            & $51.6$ & $-7.6$
            & $62.1$ & $-5.5$ \\

            \quad + LAPS
            & $\mathbf{60.0}$ & $\mathbf{-1.1}$
            & $\mathbf{64.2}$ & $\mathbf{-1.5}$ \\

            \midrule

            OPR-RU
            & $54.2$ & $-4.9$
            & ${64.9}$ & $-1.2$ \\

            \quad + LAPS
            & $\mathbf{60.0}$ & $\mathbf{-2.0}$
            & $\mathbf{65.5}$ & $\mathbf{-0.4}$ \\

            \bottomrule
        \end{tabular}%
        }
    \end{minipage}
\end{table*}

\noindent\textbf{Effect of the Replay Ratio.}
Figure~\ref{fig:replay_ratio} compares LAPS with replay-based methods under
different replay ratios, testing whether retention is mainly explained by
the number of previous task examples. Increasing the replay ratio does not
lead to monotonic improvements, and LAPS remains competitive with a much
smaller replay budget. These results suggest that retention depends not only
on replay volume but also on how historical samples are used for alignment
and prompt space learning.

\noindent\textbf{Effect of the Soft Prompt Length.}
Table~\ref{tab:soft_prompt_number} varies the number of soft prompt tokens to
test whether the observed gains depend on a particular prompt length.
Increasing the prompt length does not consistently improve performance, and
the baselines show greater sensitivity to this choice. Across all tested
prompt lengths and model scales, LAPS maintains the highest Average,
suggesting that its aggregate advantage does not depend on a narrowly tuned
prompt length.

\noindent\textbf{Effect of the B\'ezier Degree $k$.}
Table~\ref{tab:bezier_degree} studies the capacity of the learned prompt space
by varying the B\'ezier degree. At $k=1$, prompt search is limited to convex
combinations of the task anchors. Increasing $k$ introduces learnable control
prompts and consistently improves both Average and BWT up to $k=3$ across model
scales. Further increasing $k$ to $4$ brings no additional gain, suggesting that a moderate degree
is sufficient.

\noindent\textbf{Compatibility of LAPS with On-Policy Replay.}
LAPS can be combined with on-policy replay methods. Table~\ref{tab:LAPS_with_onpolicy_methods}
shows that adding LAPS improves both Average and BWT for OPR-SC and OPR-RU
across the reported model scales. These results indicate good compatibility with on-policy replay: replay selects which historical examples are revisited, while
LAPS provides an alignment and prompt selection mechanism for using the selected examples.

\section{Conclusion and Discussion}
We have presented LAPS for continual LLM adaptation by jointly updating prompts and model parameters. LAPS aligns historical prompts with the evolving model
through self-distillation, learns an anchored prompt space that captures
cross-task relationships, and selects task-specific soft prompts through
validation-guided prompt selection. Experiments on TRACE across three Qwen3 scales show
that LAPS achieves the highest average performance among compared methods while
reducing forgetting. These results suggest that organizing task-specific soft prompts into a structured,
searchable space can facilitate cross-task knowledge transfer during continual
adaptation.

The current formulation assumes known task identities and introduces additional
control prompts as tasks accumulate. This becomes more challenging in open-ended learning, where task boundaries
are unavailable and prompt count can grow continuously. A natural extension is to learn a task-free router
that selects prompts from the input, without explicit task labels.
Meanwhile, prompt capacity could be managed by adding new control
prompts only when needed and consolidating redundant ones, thereby supporting flexible cross-task interactions while keeping memory
overhead bounded.

The learned anchored prompt space also provides a way to study which previous tasks are
useful for a new task. By evaluating how task losses change across the anchored prompt
space, LAPS could identify regions associated with positive transfer and use this information to guide prompt selection. The learned prompt
space could also provide better initializations for new task prompts by reusing
nearby anchors or control prompts, reducing the amount of data and optimization
needed for adaptation. These directions would enable LAPS to reuse cross-task structure more explicitly
and improve the efficiency of future adaptation.

A deeper question concerns which historical behaviors should be preserved and
how much preservation prompt alignment can achieve. Distillation from the
previous model provides continuity, but may also retain errors or outdated
knowledge. Future work could distinguish useful capabilities from behaviors
that should change, allowing retention and correction to coexist. In parallel,
characterizing when prompt updates can compensate for backbone changes---and
when they cannot---would clarify the limits of prompt-based recovery and help
determine when continual learners require changes to the underlying model.

\clearpage
\newpage
\bibliographystyle{assets/plainnat}
\bibliography{ref}
\clearpage
\newpage
\appendix

\section{Bayesian Optimization for Validation-Guided Prompt Selection}
\label{app:bo_details}

In this section, we describe how LAPS searches the final anchored
prompt space. Once the model, task anchors, and control prompts
are fixed, the remaining task is to select a simplex coordinate
for each observed task.
LAPS treats the task-specific validation score as a black-box objective and
optimizes it with Bayesian optimization (BO,~\citeauthor{jones1998efficient},~\citeyear{jones1998efficient}).
On TRACE~\citep{wang2023trace}, the search variable is
\begin{equation}
\bm{\lambda}
=
(\lambda_1,\ldots,\lambda_8)
\in\Delta^7,
\qquad
\lambda_j\geq0,
\qquad
\sum_{j=1}^{8}\lambda_j=1.
\label{eq:bo_simplex}
\end{equation}
Given $\bm{\lambda}$, the cubic B\'ezier simplex deterministically constructs
$\mathbf{Z}_{T}(\bm{\lambda})$.

\noindent\textbf{Task-wise search.}
We run BO independently for each task:
\begin{equation}
\bm{\lambda}_{j,T}^{\star}
=
\argmax_{\bm{\lambda}\in\Delta^7}
\mathrm{Val}_{j,T}(\bm{\lambda}),
\label{eq:app_bo_objective}
\end{equation}
where $\mathrm{Val}_{j,T}(\bm{\lambda})$ is the validation score of
$\mathbf{Z}_{T}(\bm{\lambda})$ on task $\mathcal{T}_j$, computed with the corresponding evaluation metric (see
Table~\ref{tab:dataset_details}).

\noindent\textbf{Initial design.}
Each run starts from $16$ simplex coordinates. The first eight are the task
vertices, which assign all weight to one task at a time. The ninth is the
simplex centroid $\bm{\lambda}_{\mathrm{centroid}}=(1/8,\ldots,1/8)$, which gives
equal weight to all tasks. The remaining seven are well-spread mixed-task
coordinates obtained by mapping Sobol points~\citep{sobol1967distribution}
to the simplex, providing a small deterministic set of interior mixtures
before BO starts proposing new coordinates. Including all task vertices ensures
that the original task-specific prompts are part of the search.

\noindent\textbf{Candidate evaluation.}
For each candidate $\bm{\lambda}$, we construct
$\mathbf{Z}_{T}(\bm{\lambda})$ and evaluate it with the task-specific metric on
at most $500$ validation examples using temperature sampling with
temperature $0.1$. Each candidate is evaluated over three independent
inference runs. The empirical mean is used as the BO
observation, and the empirical variance estimates the observation noise.

\noindent\textbf{Gaussian process on simplex coordinates.}
We fit a Gaussian process (GP) surrogate to
$\mathrm{Val}_{j,T}(\bm{\lambda})$. Initially, each simplex
coordinate is transformed elementwise as
\begin{equation}
\widetilde{\bm{\lambda}}
=
\sqrt{\bm{\lambda}}
=
\left(
\sqrt{\lambda_1},\ldots,\sqrt{\lambda_8}
\right).
\label{eq:sqrt_simplex}
\end{equation}
Since $\sum_j\lambda_j=1$, the transformed coordinate satisfies
$\|\widetilde{\bm{\lambda}}\|_2=1$ and lies in the positive orthant of the unit
sphere. Euclidean distance in this square-root representation is proportional
to the Hellinger distance between simplex coordinates.

We use a constant kernel multiplied by a radial basis function kernel:
\begin{equation}
K(
\bm{\lambda},
\bm{\lambda}'
)
=
c
\exp
\left(
-
\frac{
\left\|
\sqrt{\bm{\lambda}}
-
\sqrt{\bm{\lambda}'}
\right\|_2^2
}{
2r^2
}
\right),
\label{eq:bo_kernel}
\end{equation}
where $c$ and $r$ are the kernel amplitude and bandwidth, respectively. We initialize
$c=1.0$ and $r=0.5$ and optimize both hyperparameters by maximizing the GP log
marginal likelihood with ranges $c\in[0.1,10]$ and
$r\in[0.05,5.0]$. The empirical variance from repeated inference runs is
provided to the GP as heteroscedastic observation noise.

\noindent\textbf{Expected improvement.}
At each BO iteration, we generate $4{,}096$ candidate simplex coordinates. Half
are sampled from a standard Dirichlet distribution to cover the simplex, one
quarter from a sparse Dirichlet distribution to emphasize boundaries and
vertices, and the remaining quarter near the current best coordinate for local
refinement.

Let $\mu_j(\bm{\lambda})$ and $s_j(\bm{\lambda})$ denote the GP predictive mean
and standard deviation for task $\mathcal{T}_j$, and let $u_j^{+}$ denote the
best validation score observed so far. Candidates are ranked using expected
improvement~\citep{jones1998efficient}:
\begin{align}
\operatorname{EI}_j(\bm{\lambda})
&=
s_j(\bm{\lambda})\phi(z_j(\bm{\lambda})) +
\left(\mu_j(\bm{\lambda})
-u_j^{+}
-
\xi
\right)
\Phi(z_j(\bm{\lambda})),
\label{eq:expected_improvement}
\\
z_j(\bm{\lambda})
&=
\frac{
\mu_j(\bm{\lambda})
-u_j^{+}
-
\xi
}{
s_j(\bm{\lambda})
},
\label{eq:ei_z}
\end{align}
where $\phi(\cdot)$ and $\Phi(\cdot)$ are the standard Gaussian probability density function and cumulative distribution function, respectively. We set $\xi=0.05$
and enforce $\operatorname{EI}_j(\bm{\lambda})=0$ whenever
$s_j(\bm{\lambda})=0$.

Each BO iteration evaluates up to four new, non-duplicate coordinates. Each
task receives at most $100$ evaluations, except MeetingBank with $50$. Search
stops early if the maximum $\operatorname{EI}_j(\bm{\lambda})$ remains below $0.005$ for three consecutive
iterations.

\begin{table}[t]
    \centering
    \caption{Dataset statistics for the eight TRACE continual adaptation tasks.}
    \label{tab:dataset_details}

    \footnotesize
    \setlength{\tabcolsep}{4pt}
    \renewcommand{\arraystretch}{1.05}

    \begin{tabular*}{\linewidth}{@{\extracolsep{\fill}}lccccc@{}}
        \toprule
        \textbf{Dataset}
        & \textbf{Source}
        & \textbf{Average Length}
        & \textbf{Metric}
        & \textbf{Language}
        & \textbf{Size} \\
        \midrule

        C-STANCE    & Social Media & 109  & Accuracy        & Chinese & 5,000 \\
        FOMC        & Finance      & 76   & Accuracy        & English & 5,000 \\
        MeetingBank & Meeting      & 3,550& ROUGE-L         & English & 5,000 \\
        Py150       & GitHub       & 814  & Edit Similarity & English & 5,000 \\
        ScienceQA   & Science      & 285  & Accuracy        & English & 5,000 \\
        NumGLUE-cm  & Math         & 50   & Accuracy        & English & 5,000 \\
        NumGLUE-ds  & Math         & 39   & Accuracy        & English & 5,000 \\
        20Minuten   & News         & 739  & SARI            & German  & 5,000 \\

        \bottomrule
    \end{tabular*}
\end{table}

\noindent\textbf{Final candidate.}
After BO terminates, we retain the three coordinates with the highest observed
validation scores. If the target task vertex is not among them, we add it back as an
additional finalist. Each finalist is reevaluated on the complete validation set
using eight fresh inference runs. The coordinate with the highest mean
validation score is selected as $\bm{\lambda}_{j,T}^{\star}$, and
$\mathbf{Z}_{T}(\bm{\lambda}_{j,T}^{\star})$ is used as the final soft prompt
for task $\mathcal{T}_{j}$.

\begin{table*}[t!]
    \centering
    \caption{Trainable components and parameter costs across Qwen3 model scales.}
    \label{tab:trainable_parameters}

    \footnotesize
    \setlength{\tabcolsep}{4pt}
    \renewcommand{\arraystretch}{1.05}

    \begin{tabular*}{\textwidth}{@{\extracolsep{\fill}}lcccrrr@{}}
        \toprule
        \multirow{2}{*}{\textbf{Method}}
        & \multirow{2}{*}{\textbf{Replay}}
        & \multirow{2}{*}{\textbf{Backbone}}
        & \multirow{2}{*}{\textbf{Soft Prompt}}
        & \multicolumn{3}{c}{\textbf{\# Trainable Parameters (M)}} \\
        \cmidrule(lr){5-7}
        & & & & \textbf{Qwen3-1.7B}
        & \textbf{Qwen3-4B}
        & \textbf{Qwen3-8B} \\
        \midrule

        SFT
& \xmark & \cmark & \xmark
& $1{,}720.6$
& $4{,}022.5$
& $8{,}190.7$ \\

Replay
& \cmark & \cmark & \xmark
& $1{,}720.6$
& $4{,}022.5$
& $8{,}190.7$ \\

LwF
& \xmark & \cmark & \xmark
& $1{,}720.6$
& $4{,}022.5$
& $8{,}190.7$ \\

OPR-RU
& \cmark & \cmark & \xmark
& $1{,}720.6$
& $4{,}022.5$
& $8{,}190.7$ \\

OPR-SC
& \cmark & \cmark & \xmark
& $1{,}720.6$
& $4{,}022.5$
& $8{,}190.7$ \\

ProgPrompt
& \xmark & \xmark & \cmark
& $0.5$
& $0.7$
& $1.0$ \\

ProMoT
& \xmark & \cmark & \cmark
& $1{,}721.1$
& $4{,}023.1$
& $8{,}191.7$ \\

\midrule

{LAPS (Ours)}
& \cmark & \cmark & \cmark
& ${1{,}727.9}$
& ${4{,}031.7}$
& ${8{,}205.5}$ \\
        \bottomrule
    \end{tabular*}
\end{table*}

\section{Trainable Components and Parameter Costs}\label{chara_para_count}

Table~\ref{tab:trainable_parameters} summarizes the trainable components and
parameter cost of each method. LAPS updates the backbone and soft prompts while
adding only a small prompt overhead: $7.3$M, $9.2$M, and $14.8$M
parameters for Qwen3-1.7B, Qwen3-4B, and Qwen3-8B, corresponding to
$0.43\%$, $0.23\%$, and $0.18\%$ of the backbone sizes.

\section{Dataset Details}
\label{app:dataset_details}

We conduct experiments on the eight TRACE tasks, which cover multiple domains,
languages, and evaluation metrics. Samples longer than the maximum sequence
length of $2{,}048$ tokens are removed. Table~\ref{tab:dataset_details}
summarizes the task statistics and official task-specific metrics, including
Accuracy, ROUGE-L, Edit Similarity, and SARI~\citep{xu2016optimizing}.

Specifically, following the notation in Sec.~\ref{sec:problem_setup}, let $(\bm{x},\bm{y})$
denote an evaluation example and $\hat{\bm{y}}_{j,t}(\bm{x})$ the output generated from $\bm{x}$ for task
$\mathcal{T}_j$ after learning stage $t$. For one inference run, the task score is
\begin{equation}
R_{j,t}=\frac{1}{|\mathcal{E}_j|}
\sum_{(\bm{x},\bm{y})\in\mathcal{E}_j}
m_j\!\left(\bm{x},\hat{\bm{y}}_{j,t}(\bm{x}),\bm{y}\right),
\label{eq:task_metric_average}
\end{equation}
where $\mathcal{E}_j$ is the test set of task $j$ and $m_j$
is the task-specific per-example scoring function on a $0$--$100$ scale.
The explicit input argument $\bm{x}$ supplies the source text needed by some metrics (\eg, SARI)
to score additions, retentions, and deletions. We average these scores over the eight inference runs described in
the main text. Below, we write
$\hat{\bm{y}}$ for $\hat{\bm{y}}_{j,t}(\bm{x})$ to simplify notation.

For classification and numerical reasoning, Accuracy assigns
$100\,\mathbf{1}[a_j(\hat{\bm{y}})=a_j(\bm{y})]$ to each example, where $a_j$ extracts
the task's answer label or numerical answer using the evaluation protocol.
For MeetingBank, we use the ROUGE-L $F_1$ score:
\begin{equation}
\mathrm{ROUGE\text{-}L}(\hat{\bm{y}},\bm{y})
=100\,\frac{2\,\operatorname{LCS}(\hat{\bm{y}},\bm{y})}{|\hat{\bm{y}}|+|\bm{y}|},
\label{eq:rouge_l_metric}
\end{equation}
where $\operatorname{LCS}(\cdot)$ is the length of the longest common subsequence
in evaluation tokens, and $|\cdot|$ denotes the number of evaluation tokens in the
sequence.

\begin{figure}[t]
    \centering
    \includegraphics[width=0.92\linewidth]{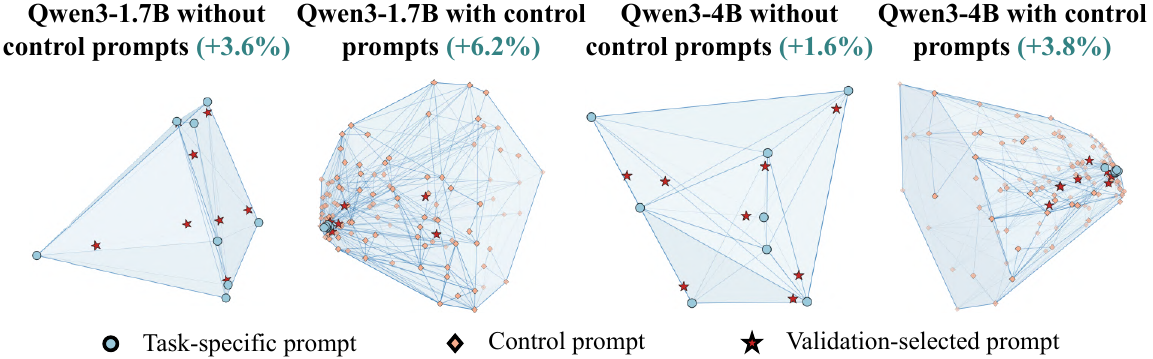}
    \caption{
    \textbf{Learned anchored prompt spaces with and without non-vertex control prompts using
    Qwen3-1.7B and Qwen3-4B.} Green values report the sum of performance
    improvements across the eight tasks after validation-guided prompt selection, measured
    relative to the corresponding task anchor prompts.
    }\label{fig:control_geometry_comparison}
\end{figure}

For Py150, Edit Similarity is the character-level matching ratio after code-literal placeholders have been normalized:
\begin{equation}
\mathrm{ES}(\hat{\bm{y}},\bm{y})
=\operatorname{round}\!\left(
100\,\frac{2\mathrm{match}(\hat{\bm{y}},\bm{y})}{|\hat{\bm{y}}|+|\bm{y}|}\right),
\label{eq:edit_similarity_metric}
\end{equation}
where $\mathrm{match}(\cdot)$ computes the total length of the matching blocks returned by the
evaluator's \texttt{SequenceMatcher} routine, and lengths are measured in
characters of the normalized strings. Two empty strings receive a score of
$100$; a single empty string receives $0$.

For 20Minuten, SARI evaluates the additions, retentions, and deletions in the
model-generated simplification $\hat{\bm{y}}$ relative to the source text
$\bm{x}$, using the reference simplification $\bm{y}$ to assess these edits:
\begin{equation}
\mathrm{SARI}(\bm{x},\hat{\bm{y}},\bm{y})
=\frac{100}{12}\sum_{n=1}^{4}
\left(F_{\mathrm{add}}^{(n)}+F_{\mathrm{keep}}^{(n)}
+P_{\mathrm{delete}}^{(n)}\right).
\label{eq:sari_metric}
\end{equation}
Here $F_{\mathrm{add}}^{(n)}$ and $F_{\mathrm{keep}}^{(n)}$ are the $F_1$
scores for added and retained $n$-grams, respectively, and
$P_{\mathrm{delete}}^{(n)}$ is the precision of deleted $n$-grams, with
reference-supported operations determined from the source and target.
We follow the evaluator's tokenization, $n$-gram counting,
and empty denominator conventions. Higher values indicate better performance
for all four metrics.

The benchmark contains both classification and generation tasks and spans
Chinese, English, and German. This diversity allows us to evaluate continual
adaptation under heterogeneous task distributions rather than within a
single domain or output format.

\begin{table*}[t]
    \centering
    \caption{
    \textbf{Task-specific simplex coordinates selected by BO.}
    Columns follow the TRACE task order used in the main text.
    }
    \label{tab:selected_lambda}
    \footnotesize
    \setlength{\tabcolsep}{4pt}
    \renewcommand{\arraystretch}{1.05}

    \begin{tabular*}{\textwidth}{@{\extracolsep{\fill}}llcccccccc@{}}
        \toprule
        \textbf{Model}
        & \textbf{Target Task}
        & ${\lambda}_1$
        & ${\lambda}_2$
        & ${\lambda}_3$
        & ${\lambda}_4$
        & ${\lambda}_5$
        & ${\lambda}_6$
        & ${\lambda}_7$
        & ${\lambda}_8$ \\
        \midrule

        \multirow{8}{*}{Qwen3-1.7B}
        & C-STANCE
        & $0.3313$ & $0.2977$ & $0.0001$ & $0.0134$ & $0.0646$ & $0.2884$ & $0.0000$ & $0.0045$ \\
        & FOMC
        & $0.2173$ & $0.3009$ & $0.0164$ & $0.0002$ & $0.4052$ & $0.0170$ & $0.0000$ & $0.0430$ \\
        & MeetingBank
        & $0.0000$ & $0.0011$ & $0.9967$ & $0.0000$ & $0.0011$ & $0.0000$ & $0.0003$ & $0.0006$ \\
        & Py150
        & $0.0000$ & $0.0000$ & $0.0000$ & $0.9962$ & $0.0004$ & $0.0014$ & $0.0000$ & $0.0015$ \\
        & ScienceQA
        & $0.0112$ & $0.0000$ & $0.0000$ & $0.0000$ & $0.9445$ & $0.0012$ & $0.0007$ & $0.0424$ \\
        & NumGLUE-cm
        & $0.0004$ & $0.0260$ & $0.0016$ & $0.0059$ & $0.0087$ & $0.8868$ & $0.0636$ & $0.0070$ \\
        & NumGLUE-ds
        & $0.0002$ & $0.0042$ & $0.0000$ & $0.0168$ & $0.0317$ & $0.0000$ & $0.9471$ & $0.0001$ \\
        & 20Minuten
        & $0.1403$ & $0.2323$ & $0.0019$ & $0.0368$ & $0.0010$ & $0.1004$ & $0.0705$ & $0.4168$ \\
        \midrule

        \multirow{8}{*}{Qwen3-4B}
        & C-STANCE
        & $0.0767$ & $0.0797$ & $0.0957$ & $0.4814$ & $0.0292$ & $0.0118$ & $0.1497$ & $0.0759$ \\
        & FOMC
        & $0.0000$ & $0.9310$ & $0.0004$ & $0.0000$ & $0.0001$ & $0.0084$ & $0.0575$ & $0.0027$ \\
        & MeetingBank
        & $0.0249$ & $0.0000$ & $0.9634$ & $0.0000$ & $0.0000$ & $0.0000$ & $0.0032$ & $0.0085$ \\
        & Py150
        & $0.0607$ & $0.0001$ & $0.0379$ & $0.8538$ & $0.0054$ & $0.0408$ & $0.0000$ & $0.0012$ \\
        & ScienceQA
        & $0.0220$ & $0.0000$ & $0.0002$ & $0.0048$ & $0.9268$ & $0.0079$ & $0.0043$ & $0.0339$ \\
        & NumGLUE-cm
        & $0.1174$ & $0.2825$ & $0.0670$ & $0.0016$ & $0.0086$ & $0.5036$ & $0.0063$ & $0.0129$ \\
        & NumGLUE-ds
        & $0.0002$ & $0.0345$ & $0.0000$ & $0.0160$ & $0.0001$ & $0.0025$ & $0.9466$ & $0.0002$ \\
        & 20Minuten
        & $0.1589$ & $0.2461$ & $0.1227$ & $0.0508$ & $0.0356$ & $0.0511$ & $0.1885$ & $0.1464$ \\
        \midrule

        \multirow{8}{*}{Qwen3-8B}
        & C-STANCE
        & $0.0660$ & $0.2528$ & $0.1278$ & $0.0516$ & $0.4995$ & $0.0006$ & $0.0010$ & $0.0006$ \\
        & FOMC
        & $0.0002$ & $0.9137$ & $0.0004$ & $0.0002$ & $0.0000$ & $0.0465$ & $0.0000$ & $0.0389$ \\
        & MeetingBank
        & $0.0003$ & $0.0002$ & $0.9884$ & $0.0000$ & $0.0002$ & $0.0000$ & $0.0080$ & $0.0029$ \\
        & Py150
        & $0.0372$ & $0.0000$ & $0.0034$ & $0.9253$ & $0.0198$ & $0.0085$ & $0.0041$ & $0.0017$ \\
        & ScienceQA
        & $0.0001$ & $0.0005$ & $0.0000$ & $0.0323$ & $0.9030$ & $0.0639$ & $0.0002$ & $0.0000$ \\
        & NumGLUE-cm
        & $0.0830$ & $0.0588$ & $0.0211$ & $0.1582$ & $0.0070$ & $0.3813$ & $0.2833$ & $0.0073$ \\
        & NumGLUE-ds
        & $0.0156$ & $0.0054$ & $0.0223$ & $0.0002$ & $0.0003$ & $0.0001$ & $0.9208$ & $0.0353$ \\
        & 20Minuten
        & $0.0009$ & $0.0254$ & $0.0000$ & $0.0012$ & $0.1002$ & $0.0000$ & $0.1368$ & $0.7355$ \\

        \bottomrule
    \end{tabular*}
\end{table*}

\section{More Experiments}\label{additionalexp}
In this section, we present additional analyses of LAPS. We first examine how control prompts shape the prompt space,
the task-specific coordinates selected by BO, and the
convergence of BO. We then evaluate robustness to task order and
continual adaptation dynamics across model scales. Finally, we provide
qualitative examples illustrating how prompt selection improves model outputs.

\subsection{Effect of Learnable Control Prompts on Prompt Space Geometry}
\label{app:control_geometry}

To assess the contribution of non-vertex control prompts, we compare anchored
prompt spaces learned with and without these prompts using Qwen3-1.7B and Qwen3-4B
(see Fig.~\ref{fig:control_geometry_comparison}). Without control prompts, the
space reduces to the convex simplex determined by the task-specific prompts, so
search is restricted to fixed interpolation among task anchors. Adding control
prompts reshapes the intermediate regions of the space and yields a larger
total performance improvement across the eight tasks.

\subsection{Task-Specific Coordinates Selected by BO}
\label{app:selected_lambda}

Table~\ref{tab:selected_lambda} reports the simplex coordinates selected by BO
for every task. The coordinates show two patterns. For MeetingBank, Py150,
ScienceQA, and NumGLUE-ds, the selected coordinate often stays close to the
corresponding task vertex, suggesting that the original task direction remains
most useful. For C-STANCE, NumGLUE-cm, and 20Minuten, the selected simplex
coordinates assign substantial weight to multiple task anchors, indicating
greater benefit from combining task directions.
The coordinates also vary across model scales, supporting task- and
model-specific prompt search rather than one fixed combination.

\begin{figure*}[t]
    \centering
    \includegraphics[
        width=\textwidth
    ]{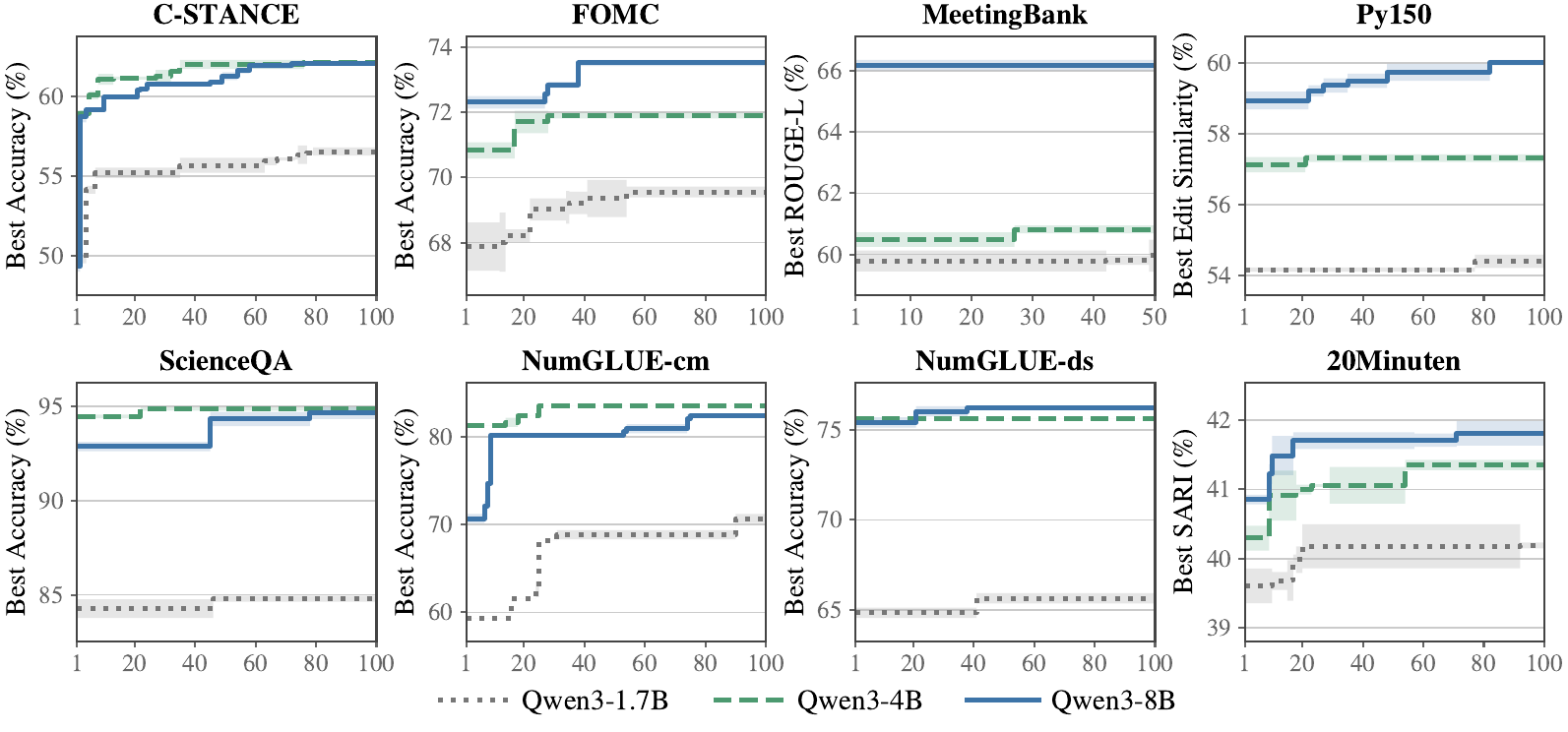}
    \caption{
        \textbf{BO convergence across eight TRACE tasks and three Qwen3 model scales.}
        Each curve reports the best validation performance found so far as the
        evaluation budget increases.
    }
    \label{fig:bo_convergence}
\end{figure*}

\begin{table*}[t]
    \centering
    \caption{
    \textbf{Continual adaptation performance under the reverse task order}: 20Minuten $\rightarrow$ NumGLUE-ds $\rightarrow$ NumGLUE-cm
    $\rightarrow$ ScienceQA $\rightarrow$ Py150 $\rightarrow$ MeetingBank
    $\rightarrow$ FOMC $\rightarrow$ C-STANCE.
    }
    \label{tab:reverse_order}

    \scriptsize
    \setlength{\tabcolsep}{3.2pt}
    \renewcommand{\arraystretch}{1.10}

    \resizebox{\textwidth}{!}{%
    \begin{tabular}{lcccccccc|cc}
        \toprule
        \textbf{Method}
        & \textbf{20Minuten}
        & \textbf{NG-ds}
        & \textbf{NG-cm}
        & \textbf{ScienceQA}
        & \textbf{Py150}
        & \textbf{MeetingBank}
        & \textbf{FOMC}
        & \textbf{C-STANCE}
        & \textbf{Average} $\uparrow$
        & \textbf{BWT} $\uparrow$ \\
        \midrule

        \rowcolor{gray!15}
        \multicolumn{11}{l}{\textbf{Qwen3-1.7B}} \\

        SFT
        & $39.1$
        & $\mathbf{64.2}$
        & $54.2$
        & $18.2$
        & $58.6$
        & $43.8$
        & $52.3$
        & $50.8$
        & $47.6$
        & $-17.1$ \\

        Replay
        & $39.5$
        & $59.0$
        & $\mathbf{59.3}$
        & $68.8$
        & $58.8$
        & $56.0$
        & $62.3$
        & $51.4$
        & $56.9$
        & $-7.7$ \\

        ProgPrompt
        & $38.5$
        & $49.2$
        & $37.7$
        & $70.9$
        & $50.5$
        & $47.7$
        & $60.6$
        & $50.1$
        & $50.6$
        & $\mathbf{0.0}$ \\

        ProMoT
        & $39.7$
        & $61.4$
        & $30.6$
        & $75.6$
        & $59.4$
        & $\mathbf{61.9}$
        & $\mathbf{64.7}$
        & $\mathbf{52.2}$
        & $55.7$
        & $-4.2$ \\
        \noalign{\vskip 1pt}
        \hline
        \noalign{\vskip 1pt}
        \textbf{LAPS (Ours)}
        & $\mathbf{39.8}$
        & $56.3$
        & $50.9$
        & $\mathbf{76.5}$
        & $\mathbf{60.3}$
        & $61.7$
        & $62.4$
        & $51.3$
        & $\mathbf{57.4}$
        & $-2.7$ \\

        \midrule
        \rowcolor{gray!15}
        \multicolumn{11}{l}{\textbf{Qwen3-4B}} \\

        SFT
        & $38.6$
        & $71.2$
        & $65.7$
        & $9.8$
        & $63.1$
        & $61.4$
        & $53.9$
        & $56.3$
        & $52.5$
        & $-16.3$ \\

        Replay
        & $39.1$
        & $68.7$
        & $64.4$
        & $62.3$
        & $\mathbf{64.3}$
        & $63.3$
        & $60.1$
        & $\mathbf{56.8}$
        & $59.9$
        & $-7.7$ \\

        ProgPrompt
        & $39.4$
        & $52.3$
        & $56.5$
        & $\mathbf{81.6}$
        & $57.4$
        & $49.4$
        & $61.4$
        & $54.8$
        & $56.6$
        & $\mathbf{0.0}$ \\

        ProMoT
        & $40.2$
        & $73.2$
        & $\mathbf{69.1}$
        & $40.7$
        & $63.6$
        & $\mathbf{65.5}$
        & $67.4$
        & $55.3$
        & $59.4$
        & $-8.3$ \\
        \noalign{\vskip 1pt}
        \hline
        \noalign{\vskip 1pt}
        \textbf{LAPS (Ours)}
        & $\mathbf{40.6}$
        & $\mathbf{74.6}$
        & $62.2$
        & $62.7$
        & $62.6$
        & $65.1$
        & $\mathbf{71.7}$
        & $56.7$
        & $\mathbf{62.0}$
        & $-5.7$ \\

        \bottomrule
    \end{tabular}%
    }
\end{table*}

\subsection{BO Convergence}
\label{app:bo_convergence}

Figure~\ref{fig:bo_convergence} shows BO convergence across the eight tasks and
three model scales. In most cases, the best validation performance improves
rapidly within the first few evaluations and then stabilizes, indicating that
useful regions of the learned space can be found with a modest
search budget.

\subsection{Robustness to Task Order}
\label{app:reverse_order}

To assess sensitivity to task order, we evaluate LAPS with the TRACE task
sequence reversed. Although task-wise performance varies with the learning
order, LAPS achieves the highest Average performance among the compared
methods at both model scales (see Table~\ref{tab:reverse_order}). LAPS also
substantially reduces forgetting relative to parameter updating baselines
such as SFT, Replay, and ProMoT. ProgPrompt keeps the backbone and previously
learned prompts fixed, which gives it a BWT of $0.0\%$, but this retention comes
with much lower Average performance. These results indicate that LAPS preserves
a stronger adaptation--retention balance under this different task order.

\begin{figure*}[t!]
    \centering
    \includegraphics[
        width=\textwidth
    ]{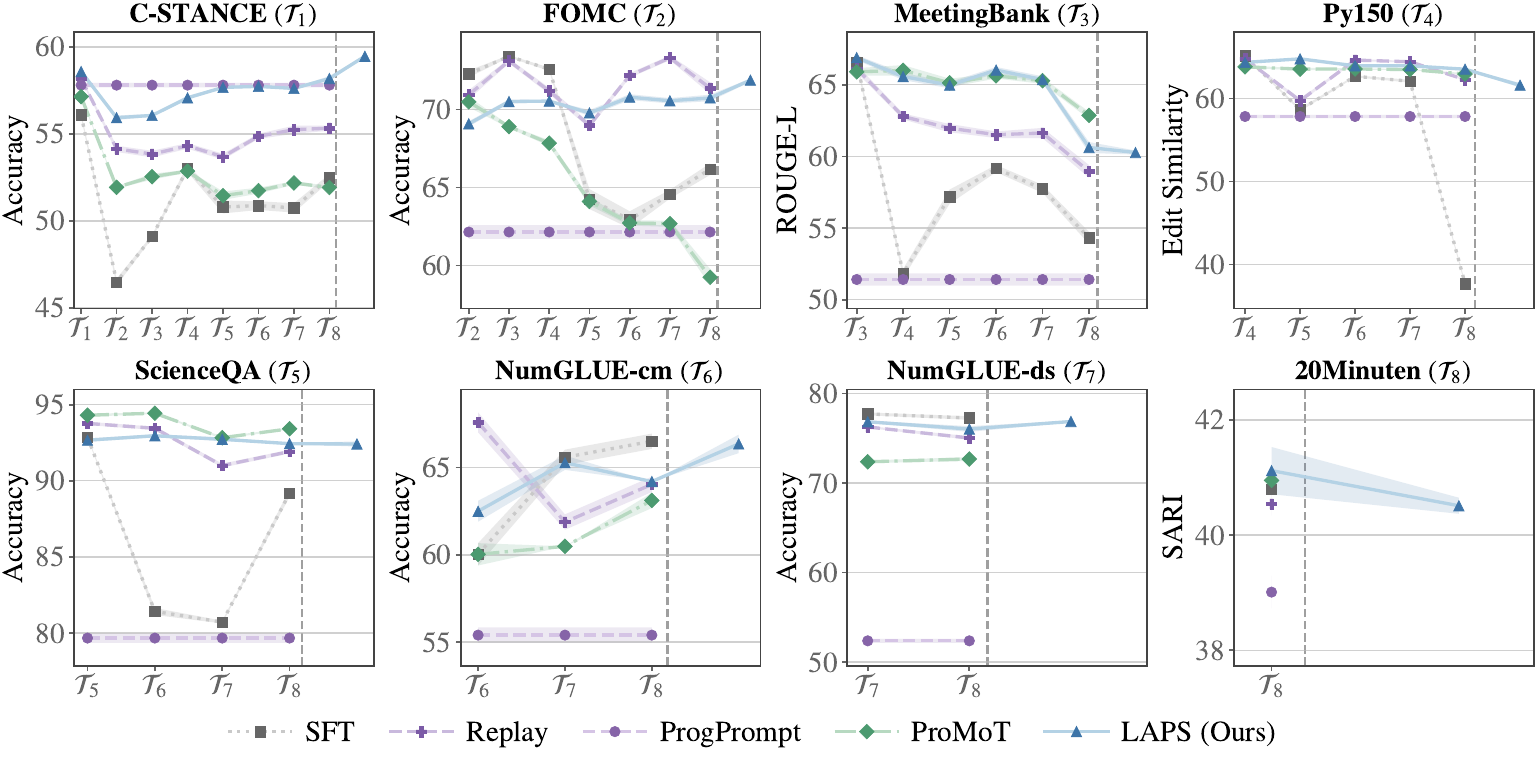}
    \caption{
        Task-wise continual adaptation trajectories on TRACE using Qwen3-4B.
    }
    \label{fig:retention_4b}
\end{figure*}
\begin{figure*}[t!]
    \centering
    \includegraphics[
        width=\textwidth
    ]{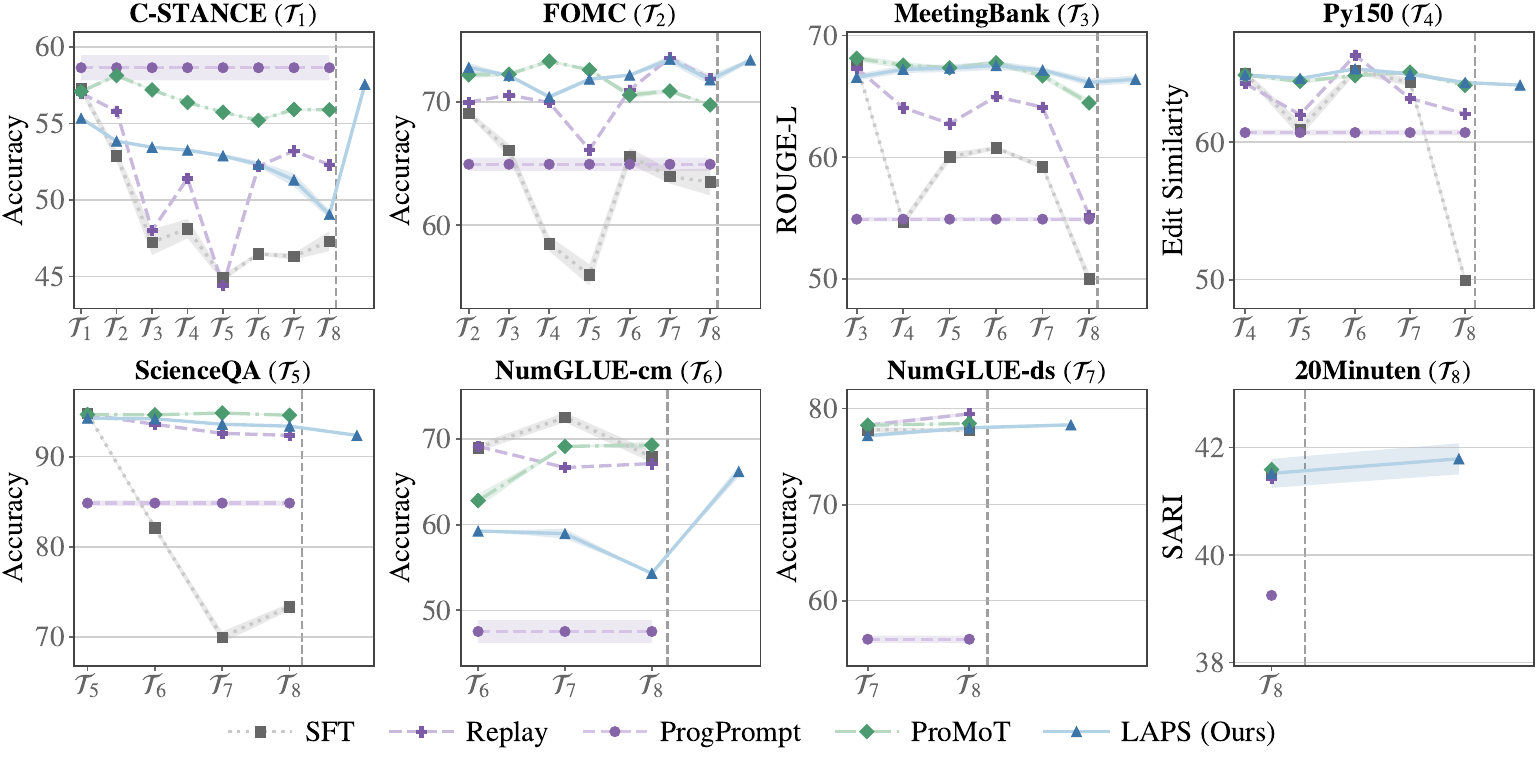}
    \caption{
        Task-wise continual adaptation trajectories on TRACE using Qwen3-8B.
    }
    \label{fig:retention_8b}
\end{figure*}

\subsection{Additional Continual Adaptation Dynamics}
\label{app:additional_dynamics}

The main text reports task-wise continual adaptation trajectories on
Qwen3-1.7B. Figures~\ref{fig:retention_4b} and
\ref{fig:retention_8b} extend this analysis to Qwen3-4B and Qwen3-8B.
Across both model scales, methods that update model parameters continue to
exhibit noticeable fluctuations on past tasks, whereas LAPS
generally maintains more stable historical performance while preserving
strong adaptation to new tasks.

\subsection{Qualitative Analysis of Validation-Guided Prompt Selection}
\label{app:qualitative_manifold}
For each example, we compare the responses generated by Qwen3-8B using each
of the eight task anchor prompts with the response generated using the prompt
selected on the corresponding task's validation set. The examples in Figs.~\ref{fig:qualitative_meetingbank} and~\ref{fig:qualitative_fomc} show that
searching the learned anchored space can recover behaviors
that are not obtained by choosing a single task vertex. In the MeetingBank
example, the MeetingBank vertex contains relevant information but over-generates
the transcript, leading to a low ROUGE-L score. The selected prompt stays close
to the MeetingBank vertex while producing a more concise summary, increasing
per-example ROUGE-L from $8.25$ to $45.28$.

For the FOMC example, all eight task vertices predict the example as neutral,
whereas the selected interior prompt gives the correct hawkish prediction.
Although the selected coordinate remains close to the FOMC vertex, small
contributions from other task directions are sufficient to change the model's
behavior.

\begin{figure}[!htbp]
\centering
\small
\begin{qualexample}{Example for MeetingBank}

\textbf{Query:}
Write a summary of the meeting transcript concerning access to low cost
veterinary care in Boston.

\medskip
\textbf{Reference:}
Order for a hearing regarding access to low cost veterinary care.
Remains in the Committee on City Services and Innovation Technology.

\medskip
\textbf{Responses:}

\renewcommand{\arraystretch}{1.08}
\setlength{\tabcolsep}{3pt}

\begin{tabularx}{\linewidth}{@{}p{0.15\linewidth}Xr@{}}
\toprule
\textbf{Prompt}
& \textbf{Generated Summary}
& \textbf{ROUGE-L} \\
\midrule

C-STANCE vertex
&
Councilor Braden proposed a hearing to address access to low cost veterinary
care for residents in public and subsidized affordable housing. The response
also discusses unequal access across neighborhoods and a One Health approach.
&
21.74
\\

FOMC vertex
&
The transcript discusses difficulty accessing low cost veterinary care,
particularly for residents in Brighton, and calls for collaboration between
city agencies and veterinary organizations.
&
16.16
\\

MeetingBank vertex
&
Order for a hearing regarding access to low cost veterinary care. The response
then reproduces a long portion of the meeting discussion, including details on
housing developments, veterinary providers, public health, and the One Health
approach.
&
8.25
\\

Py150 vertex
&
Councilor Braden called for a hearing on the difficulty faced by low income
residents in accessing affordable veterinary care and proposed considering a
holistic One Health approach.
&
20.95
\\

ScienceQA vertex
&
The hearing will address the problem of low income Boston residents lacking
access to affordable veterinary care and consider partnerships with veterinary
services.
&
28.57
\\

NumGLUE-cm vertex
&
A hearing was called to address difficulty accessing affordable veterinary
care in Boston, especially in Brighton, and to consider a One Health approach
linking human and animal health.
&
23.26
\\

NumGLUE-ds vertex
&
The hearing addresses unequal access to low cost veterinary care for residents
in public and subsidized housing and proposes a One Health framework for city
services.
&
20.56
\\

20Minuten vertex
&
A Boston City Council hearing was held on access to low cost veterinary care.
The response emphasizes unequal access and the potential use of a One Health
approach.
&
24.24
\\

\midrule

\textbf{Selected}
&
\textbf{Order for a hearing regarding Access to Low Cost Veterinary Care.
The Chair recognized Councilor Breadon. Councilor Breadon offered the
following: Order for a hearing regarding access to low cost veterinary care.}
&
\textbf{45.28}
\\

\bottomrule
\end{tabularx}

\medskip
\textbf{BO coordinate:}
\[
\bm{\lambda}^{\star}
=
[
0.0003,\,
0.0002,\,
0.9884,\,
0.0000,\,
0.0002,\,
0.0000,\,
0.0080,\,
0.0029
].
\]

\end{qualexample}
\caption{\textbf{MeetingBank summaries generated by Qwen3-8B using the eight task anchor prompts and the validation-guided prompt.} The selected prompt increases ROUGE-L from $8.25$ at the MeetingBank vertex to $45.28$.}
\label{fig:qualitative_meetingbank}
\end{figure}

\begin{figure}[p]
\centering
\small
\begin{qualexample}{Example for FOMC}

\textbf{Query:}
What is the monetary policy stance for the following text?
A. dovish, B. hawkish, C. neutral.

\medskip
\textbf{Text:}
Of course, the basic problem in tackling the inflation of the 1970s is that
we have one observation and many competing theories.

\medskip
\textbf{Reference:} B (hawkish).

\medskip
\textbf{Predictions:}

\begin{center}
\setlength{\tabcolsep}{9pt}
\begin{tabular}{lcc}
\toprule
\textbf{Prompt} & \textbf{Prediction} & \textbf{Correct} \\
\midrule
C-STANCE vertex    & C (neutral) & No \\
FOMC vertex        & C (neutral) & No \\
MeetingBank vertex & C (neutral) & No \\
Py150 vertex       & C (neutral) & No \\
ScienceQA vertex   & C (neutral) & No \\
NumGLUE-cm vertex  & C (neutral) & No \\
NumGLUE-ds vertex  & C (neutral) & No \\
20Minuten vertex   & C (neutral) & No \\
\midrule
\textbf{Selected} & \textbf{B (hawkish)} & \textbf{Yes} \\
\bottomrule
\end{tabular}
\end{center}

\medskip
\textbf{BO coordinate:}
\[
\bm{\lambda}^{\star}
=
[
0.0002,\,
0.9137,\,
0.0004,\,
0.0002,\,
0.0000,\,
0.0465,\,
0.0000,\,
0.0389
].
\]

\end{qualexample}
\caption{\textbf{FOMC predictions generated by Qwen3-8B using the eight task anchor prompts and the validation-guided prompt.} The selected prompt correctly predicts hawkish, whereas all task anchor prompts predict neutral.}
\label{fig:qualitative_fomc}
\end{figure}

\end{document}